\documentclass{article}

\usepackage{microtype}
\usepackage{graphicx}
\usepackage{subcaption}
\usepackage{xcolor}
\usepackage{enumitem}
\usepackage{booktabs} %
\usepackage{lipsum}
\usepackage{wrapfig}

\usepackage{hyperref}

\definecolor{candypink}{rgb}{0.89, 0.44, 0.48}          %
\definecolor{mediumaquamarine}{rgb}{0.4, 0.8, 0.67}     %
\definecolor{azure}{rgb}{0.0, 0.5, 1.0}                 %
\definecolor{awesome}{rgb}{1.0, 0.13, 0.32}             %

\newcommand{\pmalink}{\href{https://seohong.me/projects/pma/}{our project page}\xspace}
\newcommand{\pmavideo}{\href{https://seohong.me/projects/pma/}{videos}\xspace}
\newcommand{\pmaaddress}{{\url{https://seohong.me/projects/pma/}}}

\usepackage[accepted]{icml2023}

\usepackage{amsmath}
\usepackage{amssymb}
\usepackage{mathtools}
\usepackage{amsthm}

\usepackage[capitalize,noabbrev]{cleveref}

\theoremstyle{plain}
\newtheorem{theorem}{Theorem}[section]

\newtheorem{lemma}[theorem]{Lemma}

\theoremstyle{definition}
\newtheorem{definition}[theorem]{Definition}

\theoremstyle{remark}

\usepackage[textsize=tiny]{todonotes}

\usepackage{color}
\usepackage{amsmath}
\usepackage{amssymb}
\usepackage{multirow, varwidth}
\usepackage{dblfloatfix}
\usepackage{verbatim}
\usepackage{xspace}
\makeatletter
\newif\if@restonecol
\makeatother
\usepackage{xr}
\DeclareRobustCommand\onedot{\futurelet\@let@token\@onedot}
\def\onedot{.\xspace}
\def\eg{\emph{e.g}\onedot} 
\def\ie{\emph{i.e}\onedot} 
 
 \def\vs{\emph{vs}\onedot}

\newcommand{\cutsectionup}{\vspace{-5pt}}
\newcommand{\cutsectiondown}{\vspace{-3pt}}
\newcommand{\cutsubsectionup}{\vspace{-3pt}}
\newcommand{\cutsubsectiondown}{\vspace{-3pt}}

\usepackage{amsmath,amsfonts,bm}

\def\eqref#1{equation~\ref{#1}}

\def\1{\bm{1}}

\def\va{{\bm{a}}}

\def\vs{{\bm{s}}}

\def\vz{{\bm{z}}}

\DeclareMathAlphabet{\mathsfit}{\encodingdefault}{\sfdefault}{m}{sl}
\SetMathAlphabet{\mathsfit}{bold}{\encodingdefault}{\sfdefault}{bx}{n}

\def\gA{{\mathcal{A}}}

\def\gD{{\mathcal{D}}}

\def\gF{{\mathcal{F}}}

\def\gM{{\mathcal{M}}}
\def\gN{{\mathcal{N}}}

\def\gP{{\mathcal{P}}}

\def\gS{{\mathcal{S}}}
\def\gT{{\mathcal{T}}}

\def\gZ{{\mathcal{Z}}}

\def\sR{{\mathbb{R}}}

\def\sV{{\mathbb{V}}}

\newcommand{\E}{\mathbb{E}}

\newcommand{\TV}{D_{\mathrm{TV}}}

\DeclareMathOperator*{\argmin}{arg\,min}

\hypersetup{urlcolor=magenta}

\icmltitlerunning{Predictable MDP Abstraction for Unsupervised Model-Based RL}

\begin{document}

\twocolumn[
\icmltitle{Predictable MDP Abstraction for Unsupervised Model-Based RL}

\icmlsetsymbol{equal}{*}

\begin{icmlauthorlist}
\icmlauthor{Seohong Park}{berkeley}
\icmlauthor{Sergey Levine}{berkeley}
\end{icmlauthorlist}

\icmlaffiliation{berkeley}{University of California, Berkeley}

\icmlcorrespondingauthor{Seohong Park}{seohong@berkeley.edu}

\icmlkeywords{Machine Learning, ICML}

\vskip 0.3in
]

\printAffiliationsAndNotice{}  %

\begin{abstract}
A key component of model-based reinforcement learning (RL) is a dynamics model that predicts the outcomes of actions.
Errors in this predictive model can degrade the performance of model-based controllers,
and complex Markov decision processes (MDPs) can present exceptionally difficult prediction problems.
To mitigate this issue, we propose \textbf{predictable MDP abstraction} (\textbf{PMA}):
instead of training a predictive model on the original MDP,
we train a model on a transformed MDP
with a learned action space that only permits predictable, easy-to-model actions,
while covering the original state-action space as much as possible.
As a result, model learning becomes easier and more accurate, which allows robust, stable model-based planning or model-based RL.
This transformation is learned in an \emph{unsupervised} manner, before any task is specified by the user. Downstream tasks can then be solved with model-based control in a \emph{zero-shot} fashion, without additional environment interactions.
We theoretically analyze PMA and empirically demonstrate that
PMA leads to significant improvements over prior unsupervised model-based RL approaches in a range of benchmark environments.
Our code and videos are available at \pmaaddress
\end{abstract}

\vspace{-19pt}

\section{Introduction}
\cutsectiondown
\label{sec:intro}

The basic building block of model-based reinforcement learning (RL) algorithms is a predictive model $\hat{p}(\vs'|\vs, \va)$,
typically one that predicts the next state conditioned on the previous state and action
in the given Markov decision process (MDP).
By employing predictive models with planning or RL,
previous model-based approaches have been shown to be effective in solving
a variety of complex problems, ranging from robotics \citep{daydreamer_wu2022}
to games \citep{muzero_schrittwieser2020},
in a sample-efficient manner.

However, even small errors in a predictive model can cause a model-based RL algorithm to underperform,
sometimes catastrophically \citep{hasselt2019,hvh_jafferjee2020}.
This phenomenon, referred to as \emph{model exploitation},
happens when a controller or policy exploits these errors, picking actions that the model erroneously predicts should lead to good outcomes.
This issue is further exacerbated with long-horizon model rollouts,
which accumulate prediction errors over time,
or in complex MDPs, where accurately modeling all transitions is challenging.
Previous approaches often try to address model exploitation
by estimating model uncertainty~\citep{pets_chua2018}
or only using the model for short rollouts~\citep{steve_buckman2018,mbpo_janner2019}.

\begin{figure}[t!]
    \centering
    \includegraphics[width=0.8\linewidth]{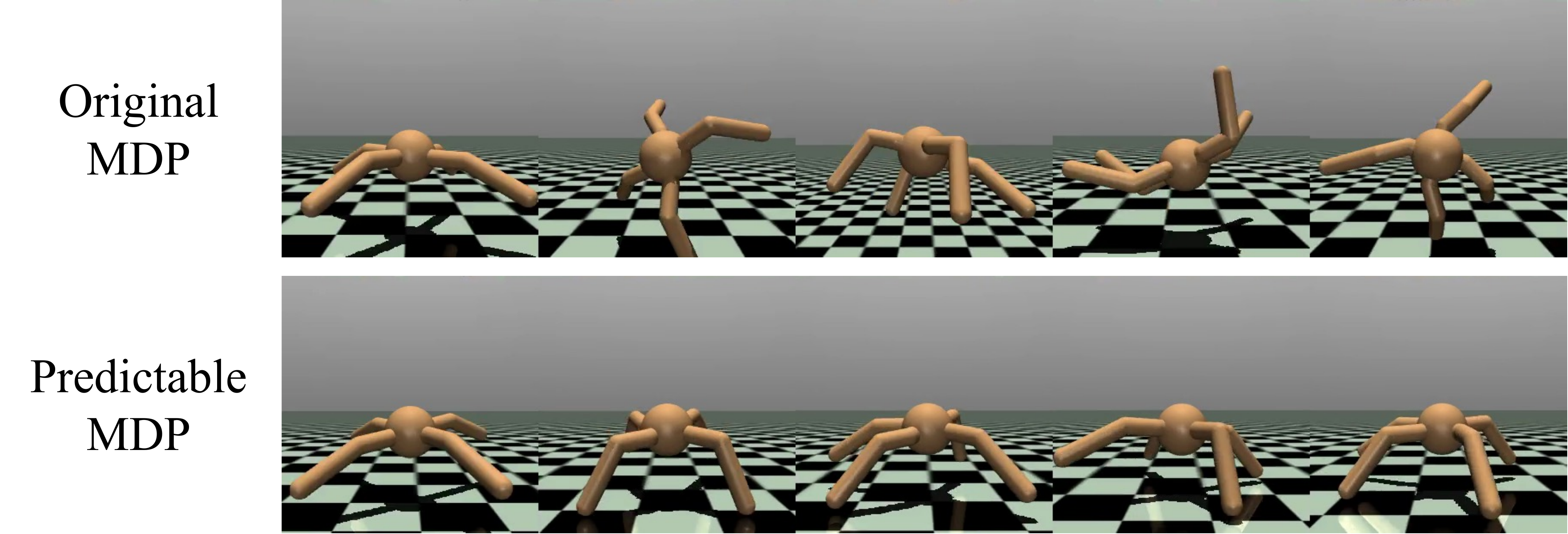}
    \vspace{-5pt}
    \caption{
    In the original Ant environment, some actions lead to unpredictable behaviors that are difficult to accurately model\footnotemark,
    which makes the learned dynamics model susceptible to catastrophic model exploitation.
    In our transformed predictable MDP, every transition is easy to model and predictable,
    which enables robust, stable model-based learning.
    }
    \label{fig:qual_ant}
    \vspace{-12pt}
\end{figure}
\footnotetext{
    Even though the Ant environment is completely deterministic,
    it is very difficult to accurately model all the possible transitions due to its complex, contact-rich dynamics (see \Cref{sec:analysis}).
}

We take a different perspective on model-based RL to tackle this challenge:
instead of training a predictive model on the original MDP,
we apply model-based RL on top of an abstracted, simplified MDP.
Namely, we first abstract the MDP into a simpler learned MDP with a transformed latent action space
by \emph{restricting} unpredictable actions,
and then build the predictive model on this simplified MDP.
Here, the transformed MDP is designed to be \emph{predictable} in the sense that every transition in the new MDP is easy to model,
while covering the original state-action space as much as possible.
As a result, there is little room for catastrophic model exploitation compared to the original, possibly complex MDP,
allowing robust model-based planning and RL.
We illustrate an example of predictable MDP transformation in \Cref{fig:qual_ant}.

We design a practical algorithm for learning predictable MDP abstractions in the setting of \emph{unsupervised} model-based RL, where the abstraction is learned in advance without any user-defined task or reward function. After unsupervised training, the learned MDP abstraction can be used to solve multiple different downstream tasks
with a model-based controller
in a \emph{zero-shot} fashion, without any environment interactions or additional model training.
The desiderata of unsupervised predictable MDP abstraction are threefold.
First, the latent actions in the transformed MDP should lead to predictable state transitions.
Second, different latent actions should lead to different outcomes.
Third, the transitions in the latent MDP should cover the original state-action space as much as possible.
In this paper, we formulate these desiderata into an information-theoretic objective and propose a practical method to optimize it.

To summarize, our main contribution in this paper is to introduce a novel perspective on model-based RL
by proposing \textbf{predictable MDP abstraction} (\textbf{PMA}) as an unsupervised model-based RL method,
which abstracts the MDP by transforming the action space to minimize model errors.
PMA can be combined with any existing model-based planning or RL method to solve downstream tasks in a zero-shot manner.
We theoretically analyze PMA and discuss when our approach can be beneficial compared to classic model-based RL.
Finally, we empirically confirm that
PMA combined with model-based RL can robustly solve a variety of tasks in seven diverse robotics environments,
significantly outperforming previous unsupervised model-based approaches.

\cutsectionup
\section{Related Work}
\cutsectiondown
\label{sec:related}

\textbf{Model-based reinforcement learning.}
Model-based RL (MBRL) involves using a predictive model that estimates the outcomes of actions in a given environment.
Previous model-based approaches utilize such learned models to maximize the reward
via planning
\citep{hernandaz1990,draeger1995,pilco_deisenroth2011,deepmpc_lenz2015,visualmpc_ebert2018,pets_chua2018,planet_hafner2019,pddm_nagabandi2019},
reinforcement learning \citep{svg_heess2015,mve_feinberg2018,steve_buckman2018,mbpo_janner2019,dreamer_hafner2020,tpc_nguyen2021},
or both \citep{mbop_argenson2021,loop_sikchi2022,tdmpc_hansen2022}.
Erroneous predictive models can yield deleterious effects on policy learning,
which is known as the model exploitation problem \citep{agnostic_ross2012,mbpo_janner2019,morel_kidambi2020,ldm_kang2022}.
To avoid making suboptimal decisions based on incorrect models,
prior works either restrict the horizon length of model rollouts \citep{mbpo_janner2019}
or employ various uncertainty estimation techniques,
such as Gaussian processes \citep{gp_rasmussen2003,pilco_deisenroth2011}
or model ensembles
\citep{epopt_rajeswaran2017,mbmpo_clavera2018,metrpo_kurutach2018,pets_chua2018,pddm_nagabandi2019,mopo_yu2020,morel_kidambi2020}.
Our work is orthogonal and complementary to these model-based RL algorithms.
We propose an action representation learning method that abstracts the MDP into one that is \emph{more predictable},
thus making model learning easier,
which can be employed in combination with a variety of existing model-based RL methods.

\textbf{Predictable behavior learning.}
Our main idea conceptually relates to prior work on predictable behavior learning.
One such work is RPC \citep{rpc_eysenbach2021}, which encourages the agent
to produce predictable behaviors by minimizing model errors.
SMiRL \citep{smirl_berseth2021} and IC2 \citep{ic2_rhinehart2021}
actively seek stable behaviors by reducing uncertainty.
While these methods incentivize the agent to behave in predictable ways or visit familiar states,
they do not aim to provide a model-based RL method,
instead utilizing the predictability bonus either for intrinsic motivation \citep{smirl_berseth2021,ic2_rhinehart2021}
or to improve robustness \citep{rpc_eysenbach2021}.
In contrast, we show that optimizing for predictability can lead to significantly more effective model-based RL performance.

\textbf{MDP abstraction and hierarchical RL.}
MDP abstraction deals with the problem of building simplified MDPs
that usually have simpler state or action spaces to make RL more tractable.
State abstraction \citep{bisim_li2006,e2c_watter2015,world_ha2018,deepmdp_gelada2019,bisim_castro2019,dreamer_hafner2020}
focuses on having a compact state representation to facilitate learning.
Temporal abstraction and hierarchical RL
\citep{option_sutton1999,option_stolle2002,oc_bacon2017,fun_vezhnevets2017,eigen_machado2017,hiro_nachum2018,diayn_eysenbach2019,ho2_wulfmeier2021,mo2_salter2022}
aim to learn temporally extended behaviors to reduce high-level decision steps.
Different from these previous approaches,
we explore a lossy approach to MDP abstraction where the action space is transformed into one that only permits more predictable transitions,
thus facilitating more effective model-based reinforcement learning.

\textbf{Unsupervised reinforcement learning.}
The goal of unsupervised RL is to acquire primitives, models, or other objects that are useful for downstream tasks through unsupervised interaction with the environment.
The process of learning a predictable MDP abstraction with our method corresponds to an unsupervised RL procedure. Prior unsupervised RL methods have used intrinsic rewards for
maximizing state entropy \citep{smm_lee2019,protorl_yarats2021,apt_liu2021},
detecting novel states \citep{icm_pathak2017,rnd_burda2019,disag_pathak2019},
learning diverse goal-condition policies \citep{skewfit_pong2020,lexa_mendonca2021},
or acquiring temporally extended skills
\citep{vic_gregor2016,diayn_eysenbach2019,dads_sharma2020,lsp_xie2020,disdain_strouse2022,lsd_park2022,cic_laskin2022}.
Notably, several unsupervised model-based approaches \citep{max_shyam2019,p2e_sekar2020,rajeswar2022}
have shown that predictive models trained via unsupervised exploration help solve downstream tasks efficiently.
However, these methods only focus on finding novel transitions (\ie, maximizing coverage),
without considering their \emph{predictability}.
Maximizing coverage without accounting for predictability can lead to model errors,
which in turn lead to model exploitation, as shown by \citet{max_shyam2019} as well as in our experiments.
Our method is also closely related to DADS \citep{dads_sharma2020},
an unsupervised skill discovery method that uses a similar mutual information (MI) objective to ours. 
However,
the main focus of our work is different from DADS:
while the goal of DADS is to acquire a set of temporally extended skills,
analogously to other works on skill discovery,
our focus is instead on transforming the action space into one that only permits predictable actions
without temporal abstraction, maximally covering the transitions in the original MDP.
We both theoretically and empirically show
that this leads to significantly better performance in a variety of model-based RL frameworks.

\cutsectionup
\section{Preliminaries and Problem Statement}
\cutsectiondown
We consider an MDP without a reward function, also referred to as a controlled Markov process (CMP),
$\gM := (\gS, \gA, \mu, p)$,
where $\gS$ denotes the state space, $\gA$ denotes the action space,
$\mu \in \gP(\gS)$ denotes the initial state distribution,
and $p: \gS \times \gA \to \gP(\gS)$ denotes the state transition distribution.
We also consider a set of $N$ downstream tasks $\gT = \{T_0, T_1, \dots, T_{N-1}\}$,
where each task corresponds to a reward function $r_i: \gS \times \gA \to \sR$ for $i \in [N]$. 
$[N]$ denotes the set of $\{0, 1, \dots, N-1\}$.
We denote the supremum of the absolute rewards as $R = \sup_{\vs \in \gS, \va \in \gA, i \in [N]} |r_i(\vs, \va)|$
and the discount factor as $\gamma$.

\textbf{Problem statement.}
In this work, we tackle the problem of \emph{unsupervised model-based RL}, which consists of two phases.
In the first unsupervised training phase,
we aim to build a predictive model in a given CMP without knowing the tasks.
In the subsequent testing phase,
we are given multiple task rewards in the same environment and aim to solve them only using the learned model
without additional training; \ie, in a zero-shot manner.
Hence, the goal is to build a model that best captures the environment so that we can later robustly employ the model to solve diverse tasks.

\cutsectionup
\section{Predictable MDP Abstraction (PMA)}
\cutsectiondown

Model-based RL methods typically learn a model $\hat{p}(\vs'|\vs, \va)$.
However, na\"{i}vely modeling all possible transitions is error-prone in complex environments,
and subsequent control methods (planning or policy optimization) can exploit these errors,
leading to overoptimistic model-based estimates of policy returns and ultimately in poor performance.
Previous works generally try to resolve this
by restricting model usage \citep{steve_buckman2018,mbpo_janner2019}
or better estimating uncertainty \citep{pets_chua2018}.
Different from previous approaches,
our solution in this work is to transform the original MDP into a \emph{predictable} latent MDP,
in which every transition is predictable.
Here, ``predictable'' also means that it is \emph{easy to model}.
Formally,
we define the unpredictability of an MDP $\gM$ as the minimum possible average model error $\epsilon$
with respect to a model class $\gF$,
a state-action distribution $d$,
and a discrepancy measure $D$:
\begin{align}
    \epsilon = \inf_{f \in \gF} \E_{(\vs, \va) \sim d(\vs, \va)}[D(p(\cdot|\vs, \va) \| f(\cdot|\vs, \va))].
\end{align}
Intuitively, this measures the irreducible model error (\ie, aleatoric uncertainty) of the environment
given the capacity of the model class $\gF$.
We note that
even in a completely deterministic environment,
there may exist an irreducible model error
if the model class $\gF$ has a finite capacity (\eg, $(512, 512)$-sized neural networks)
and the environment dynamics are complex.

After transforming the original MDP into a simplified latent MDP,
we solve downstream tasks on top of the latent MDP with model-based RL.
Since the latent MDP is trained to be maximally predictable,
there is little room for model exploitation compared to the original environment. %
We can thus later robustly employ the learned latent predictive model
to solve downstream tasks with a model-based control method.

\begin{figure}[t!]
    \centering
    \includegraphics[width=0.7\linewidth]{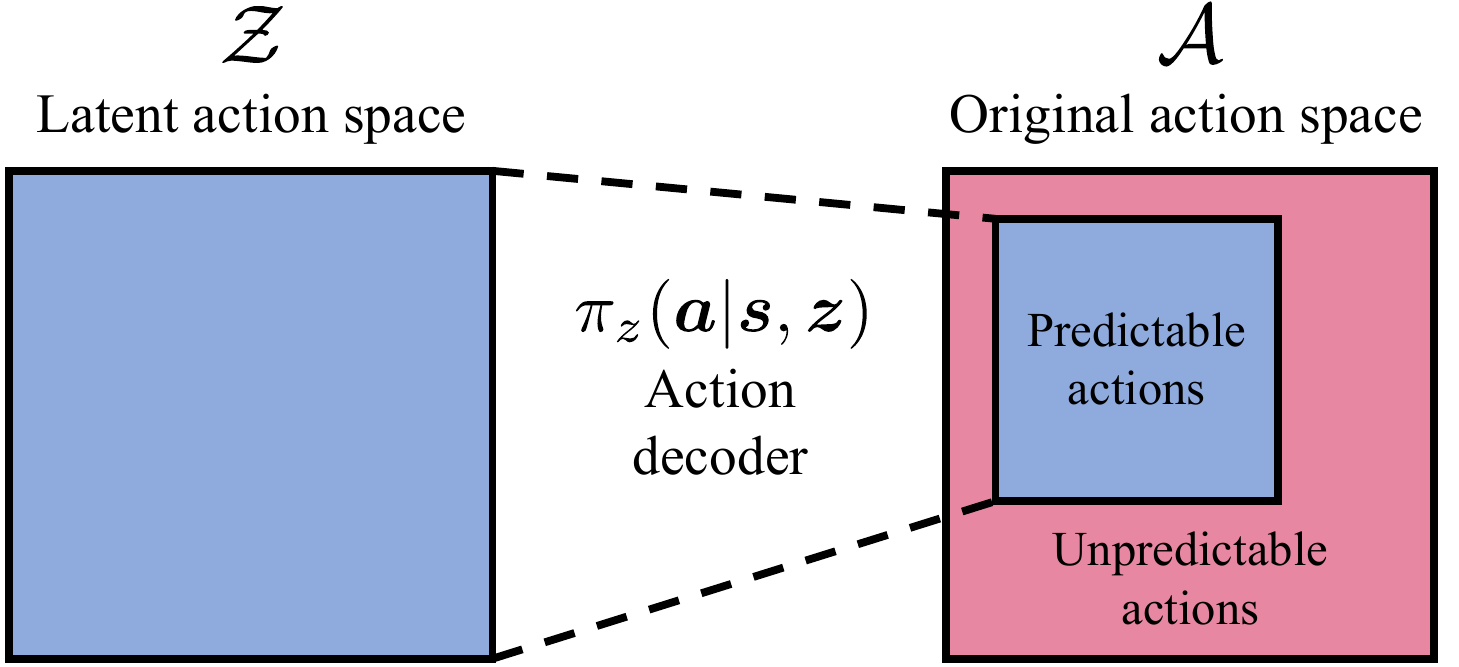}
    \vspace{-5pt}
    \caption{
        The action decoder reparameterizes the action space to only permit \emph{predictable} transitions.
    }
    \label{fig:pma_illust}
\end{figure}

\begin{figure}[t!]
    \centering
    \vspace{-5pt}
    \includegraphics[width=0.7\linewidth]{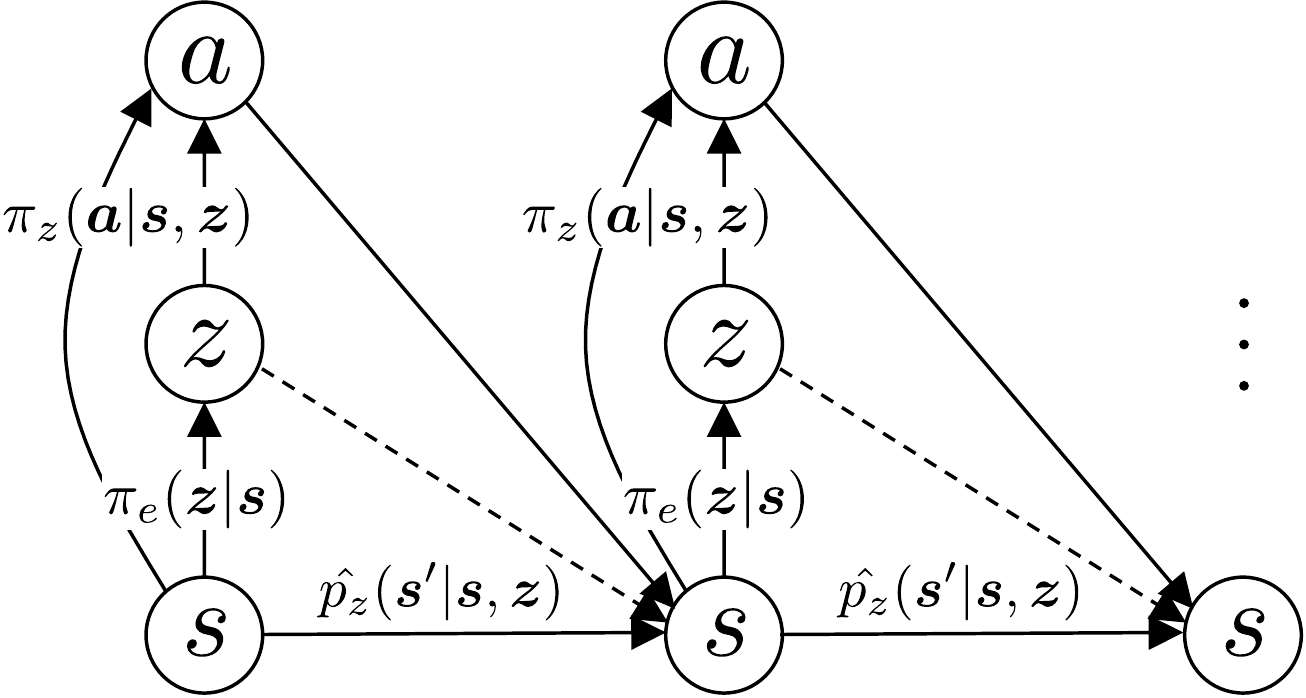}
    \vspace{-5pt}
    \caption{
        Architecture of PMA during unsupervised training.
        The exploration policy $\pi_e(\vz|\vs)$ selects a latent action,
        which is decoded into the original action space by the action decoder $\pi_z(\va|\vs, \vz)$.
        The latent model $\hat{p}_z(\vs'|\vs, \vz)$ predicts outcomes in the latent MDP.
    }
    \vspace{-12pt}
    \label{fig:pma_arch}
\end{figure}

\cutsubsectionup
\subsection{Architecture}
\cutsubsectiondown
The main idea of PMA is to \emph{restrict} the original action space in a lossy manner
so that it only permits predictable transitions.
Formally, we transform the original MDP $\gM$ into a predictable latent MDP defined as
$\gM_P := (\gS, \gZ, \mu, p_z)$
with the original state space $\gS$, the latent action space $\gZ$,
and the latent transition dynamics $p_z: \gS \times \gZ \to \gP(\gS)$.

PMA has three learnable components: an action decoder, a latent predictive model, and an exploration policy.
The \emph{action decoder} policy $\pi_z(\va|\vs, \vz)$ decodes latent actions into the original action space,
effectively reparameterizing the action space in a lossy manner (\Cref{fig:pma_illust}).
The \emph{latent predictive model} $\hat{p}_z(\vs'|\vs, \vz)$ predicts the next state in the latent MDP, %
which is \emph{jointly} trained with the action decoder to make the learned MDP as predictable as possible.
Finally, the \emph{exploration policy} $\pi_e(\vz|\vs)$ selects $\vz$'s to train the PMA's components during the unsupervised training phase,
maximizing the coverage in the original state space.
We illustrate the components of PMA in \Cref{fig:pma_arch}.

When building a latent predictive model $\hat{p}_z(\vs'|\vs, \vz; \bm{\theta})$,
we derive our objective from a Bayesian perspective,
integrating an information-theoretic representation learning goal
with information gain on the posterior over the parameters $\bm{\theta}$.
This naturally leads to the emergence of both an information-seeking exploration objective
and an MI-based representation learning method for the latent action space.

After unsupervised training,
once we get a reward function in the testing phase,
we replace the exploration policy with a \emph{task policy} $\pi(\vz|\vs)$,
which aims to select latent actions to solve the downstream task on top of our action decoder and predictive model.
This task policy can either be derived based on a planner,
or it could itself be learned via RL inside the learned model,
as we will describe in \Cref{sec:mbrl}.
Both approaches operate in \emph{zero shot},
in the sense that they do not require any additional environment interaction beyond the unsupervised phase.

\cutsubsectionup
\subsection{Objective}
\cutsubsectiondown

We now state the three desiderata of our unsupervised predictable MDP abstraction:
(\emph{i}) The latent transitions $p_z(\vs'|\vs, \vz)$ in the predictable MDP should be as predictable as possible
(\ie, minimize aleatoric uncertainty).
(\emph{ii}) The outcomes of latent actions should be as different as possible from one another
(\ie, maximize action diversity to preserve as much of the expressivity of the original MDP as possible).
(\emph{iii}) The transitions in the predictable MDP should cover the original transitions as much as possible
(\ie, encourage exploration by minimizing epistemic uncertainty).
These three goals can be summarized into the following concise information-theoretic objective:
\vspace{-2pt}
\begin{align}
    \max_{\pi_z, \pi_e} I(\bm{S}';(\bm{Z}, \bm{\Theta})|\bm{\gD}),  \label{eq:main_obj}
\end{align}
where
$\bm{\gD}$ denotes the random variable (RV) of the entire training dataset up to and including the current state $\bm{S}$,
$(\bm{S}, \bm{Z}, \bm{S}')$ denotes the RVs of $(\vs, \vz, \vs')$ tuples from the policies,
and $\bm{\Theta}$ denotes the RV of  %
the parameters of the latent predictive model $\hat{p}_z(\vs'|\vs, \vz; \bm{\theta})$.
Intuitively, this objective requires learning a latent action space, represented by $\pi_z$,
as well as an exploration policy $\pi_e$ in this latent action space,
such that at each transition the resulting next state is easy to predict from the latent action and the model parameters.
The inclusion of model parameters may seem like an unusual choice,
but it leads naturally to an information gain exploration scheme
that maximizes state coverage.

\Cref{eq:main_obj} can be decomposed as follows, revealing three terms that directly map onto our desiderata:
\vspace{-2pt}
\begin{align}
    &I(\bm{S}';(\bm{Z}, \bm{\Theta})|\bm{\gD}) \\
    =& I(\bm{S}';\bm{Z}|\bm{\gD}) + I(\bm{S}';\bm{\Theta}|\bm{\gD}, \bm{Z}) \\
    =& I(\bm{S}';\bm{Z}|\bm{S}) + I(\bm{S}';\bm{\Theta}|\bm{\gD}, \bm{Z}) \\
    =& -\underbrace{H(\bm{S}'|\bm{S},\bm{Z})}_{\text{(\emph{i}) predictability \ }}
        +\underbrace{H(\bm{S}'|\bm{S})}_{\text{(\emph{ii}) diversity}} \nonumber \\
    &+\underbrace{H(\bm{\Theta}|\bm{\gD}, \bm{Z}) -H(\bm{\Theta}|\bm{\gD}, \bm{Z}, \bm{S}')}_{\text{(\emph{iii}) information gain}}.
    \label{eq:pma_decomp}
\end{align}
The first term in \Cref{eq:pma_decomp} maximizes predictability
by reducing the entropy of the next state distribution $p_z(\vs'|\vs, \vz)$,
making the latent MDP maximally predictable.
The second term increases the entropy of the marginalized next state distribution,
effectively making the resulting states from different $z$'s different from one another.
The third term minimizes epistemic uncertainty by maximizing \emph{information gain},
the reduction in the uncertainty of the predictive model's parameters after knowing $\bm{S}'$.

With the objective in \Cref{eq:main_obj} as an intrinsic reward, %
we can optimize both the action decoder policy and exploration policy with RL.
As a result,
they will learn to produce the optimal $\vz$'s and $\va$'s that
in the long-term lead to maximal coverage of the original state space,
while making the resulting latent MDP as predictable as possible.
Also, we simultaneously train the latent predictive model $\hat{p}_z(\vs'|\vs, \vz; \bm{\theta})$
so that we can later use it for planning in the latent MDP.

\cutsubsectionup
\subsection{Practical Implementation}
\cutsubsectiondown
We now describe a practical method to optimize our main objective in \Cref{eq:main_obj}.
Since it is generally intractable to exactly estimate mutual information or information gain,
we make use of several approximations.

\textbf{Estimating $I(\bm{S'};\bm{Z}|\bm{S})$.}
First, for the first two terms in \Cref{eq:pma_decomp}, we employ a variational lower-bound approximation as follows
\citep{im_barber2003}:
\begin{align}
    &I(\bm{S}';\bm{Z}|\bm{S})
    = -H(\bm{S}'|\bm{S}, \bm{Z}) + H(\bm{S}'|\bm{S}) \\
    &\geq \E [ \log \hat{p}_z(\vs'|\vs, \vz; \bm{\phi}) - \log p_z(\vs'|\vs) ] \\
    &\approx \E [ \underbrace{\log \hat{p}_z(\vs'|\vs, \vz; \bm{\phi})
        - \log \frac{1}{L} \sum_{i=1}^{L} \hat{p}_z(\vs'|\vs, \vz_i; \bm{\phi}}_{:= r_{\text{emp}}(\vs, \vz, \vs')}) ], \label{eq:emp}
\end{align}
where $\hat{p}_z(\vs'|\vs, \vz; \bm{\phi})$ is a variational lower-bound (VLB) of $p_z(\vs'|\vs, \vz)$.
Also, we approximate the intractable marginal entropy term $H(\bm{S}'|\bm{S})$
with $L$ random samples of $\vz$'s from the latent action space \citep{dads_sharma2020}.
These approximations provide us
with a tractable intrinsic reward that can be optimized with RL.
Here, we note that the second term in \Cref{eq:emp} is a biased estimator for $\log p_z(\vs'|\vs)$,
since it is estimating an expectation inside the $\log$ with samples,
but we found this approach to still work well in practice,
and indeed multiple prior works have also explored
such a biased estimator for mutual information objectives in RL~\citep{dads_sharma2020,ibol_kim2021}.
Exploring unbiased lower bounds~\citep{mi_poole2019}
for this MI objective is an interesting direction for future work.
\textbf{Estimating information gain.}
Next, we need to estimate the information gain term in \Cref{eq:pma_decomp}.
This term could be approximated directly using prior methods that propose exploration via information gain,
\eg, using a variational approximation~\citep{vime_houthooft2016}.
In our implementation, however, we use a more heuristic approximation
that we found to be simpler to implement based on ensemble disagreement,
motivated by prior works~\citep{max_shyam2019,rp1_ball2020,p2e_sekar2020,disdain_strouse2022}.
Namely, we first approximate the model posterior with 
an ensemble of $E$ predictive models, $\{\hat{p}_z(\vs'|\vs, \vz; \bm{\theta}_i)\}_{i \in [E]}$
with $p(\bm{\theta}|\gD) = \frac{1}{E} \sum_i \delta (\bm{\theta} - \bm{\theta}_i)$.
Each component models the transitions as conditional Gaussian
with the mean given by a neural network and a unit diagonal covariance,
$\vs' \sim \gN (\bm{\mu}(\vs, \vz; \bm{\theta}_i), I)$.
We then use the variance of the ensemble means with a coefficient $\beta$, %
\begin{align}
    \E [
        \underbrace{\beta \cdot \mathrm{Tr}[\sV_i[\bm{\mu}(\vs, \vz; \bm{\theta}_i)]]}_{:= r_{\text{dis}}(\vs, \vz, \vs')}
    ], \label{eq:disag}
\end{align}
as a simple (though crude) estimator for information gain
$I(\bm{S}';\bm{\Theta}|\bm{\gD}, \bm{Z})$.
We provide a detailed justification in \Cref{sec:appx_approx}.
Intuitively, \Cref{eq:disag} encourages the agent to explore states that have not been previously visited,
where the ensemble predictions do not agree with one another.
\textbf{Training PMA.}
With these approximations,
we use \mbox{$r_{\text{emp}}(\vs, \vz, \vs') + r_{\text{dis}}(\vs, \vz, \vs')$} as an intrinsic reward
for the action decoder $\pi_z(\va|\vs, \vz)$.
For the exploration policy $\pi_e(\vz|\vs)$,
if we assume the action decoder is optimal,
we can use \mbox{$H(\bm{Z}|\bm{S}) + r_{\text{dis}}(\vs, \vz, \vs')$} as an intrinsic reward.
This can be optimized with any \emph{maximum entropy} RL method.
However, in practice, we find that
it is sufficient in our experiments to simply use a maximum entropy policy $\pi_e(\cdot|\vs) = \text{Unif}(\gZ)$
since our action decoder also maximizes $r_{\text{dis}}(\vs, \vz, \vs')$.
Finally, we fit our VLB predictive model $\hat{p}_z(\vs'|\vs, \vz; \bm{\phi})$
and ensemble models $\{\hat{p}_z(\vs'|\vs, \vz; \bm{\theta}_i)\}$ using the $(\vs, \vz, \vs')$ tuples sampled from our policies.
We describe the full training procedure of PMA in \Cref{sec:appx_training} and \Cref{alg:pma}.

\cutsubsectionup
\subsection{Connections to Prior Work}
\cutsubsectiondown
PMA's objective is related to several prior works in unsupervised RL.
For example, if we set $\beta = 0$ and \mbox{$\pi_e(\vz_t|\vs_t) = \vz_{t-1}$} for $t \geq 1$, we recover DADS \citep{dads_sharma2020},
a previous unsupervised skill discovery method.
Also, if we set \mbox{$\pi_z(\va|\vs, \vz) = \vz$}, %
PMA becomes similar to prior unsupervised model-based approaches using disagreement-based intrinsic rewards \citep{max_shyam2019,p2e_sekar2020}.
However, these methods either do not aim to cover the state-action space %
or do not consider predictability, %
which makes them suboptimal or unstable.
In \Cref{sec:exp}, we empirically compare PMA with these prior approaches
and demonstrate that our full objective makes a substantial improvement over them.
Additionally, we theoretically compare PMA with DADS in \Cref{sec:appx_pma_dads}.

\cutsectionup
\section{Model-Based Learning with PMA}
\cutsectiondown
\label{sec:mbrl}

After completing the unsupervised training of PMA,
we can employ the learned latent predictive model to optimize a reward function with model-based planning or RL.
In this section, we present several ways to utilize our predictable MDP to solve downstream tasks
in a \emph{zero-shot} manner. %

\cutsubsectionup
\subsection{Model-Based Learning with PMA}
\cutsubsectiondown
\label{sec:mbrl_ex}

After training the model, PMA can be combined with any existing model-based planning or RL method.
Specifically, we can apply any off-the-shelf model-based RL method on top of the latent action space $\gZ$
and the learned latent dynamics model $\hat{p}_z(\vs'|\vs, \vz)$ to maximize downstream task rewards,
where we use the mean of the ensemble model outputs as the latent dynamics model.
By planning over the latent action space,
we can effectively prevent model exploitation as hard-to-predict actions are filtered out.

In our experiments, we study two possible instantiations of model-based learning:
one based on model-predictive control, and one based on approximate dynamic programming (\ie, actor-critic RL),
where the learned model is used as a ``simulator'' without additional environment samples.
Note that both variants solve the new task in ``zero shot,''
in the sense that they do not require any additional collection of real transitions in the environment.

\textbf{Model predictive path integral (MPPI).}
MPPI \citep{mppi_williams2016} is %
a sampling-based zeroth-order planning algorithm
that optimizes a short-horizon sequence of (latent) actions at each time step,
executes the first action in the sequence, and then replans.
We refer to \Cref{sec:appx_training_mppi} and \Cref{alg:mppi} for the full training procedure.

\textbf{Model-based policy optimization (MBPO).}
MBPO \citep{mbpo_janner2019} is a Dyna-style model-based RL algorithm
that trains a model-free RL method
on top of truncated model-based rollouts starting from intermediate environment states.
In our zero-shot setting,
we train the task policy only using \emph{purely} model-based rollouts,
whose starting states are sampled from the restored replay buffer from unsupervised training.
We refer to \Cref{sec:appx_training_mbpo} and \Cref{alg:mbpo} for the full training procedure.

\cutsubsectionup
\subsection{Addressing Distributional Shift}
\cutsubsectiondown
\label{sec:mopo}
Using a fixed, pre-trained model for model-based control is inherently vulnerable to \emph{distributional shift}
as the test-time controller might find some ``adversarial'' $\vz$ values
that make the agent state out of distribution.
This issue applies to our zero-shot setting as well,
even though every transition in our latent MDP is trained to be predictable.
As such,
we explicitly penalize the agent for visiting out-of-distribution states,
following prior offline model-based RL methods \citep{mopo_yu2020,morel_kidambi2020},
which also deals with the same issue.
As we already have an ensemble of latent predictive models,
we use the following maximum disagreement between ensemble models \citep{morel_kidambi2020} as a penalty with a coefficient $\lambda$:
\begin{align}
    u(\vs, \vz) = -\lambda \cdot \max_{i, j \in [E]} \| \mu(\vs, \vz; \bm{\theta}_i) - \mu(\vs, \vz; \bm{\theta}_j) \|^2.
\end{align}
During task-specific planning or RL, we add this penalty to the task reward,
similarly to MOPO \citep{mopo_yu2020}.

\cutsectionup
\section{Theoretical Results}
\cutsectiondown

Predictable MDP abstraction is a lossy procedure.
In this section, we theoretically analyze the degree to which this lossiness influences the performance of a policy in the abstracted MDP,
and provide practical insights on when PMA can be useful compared to classic model-based RL.
All formal definitions and proofs can be found in \Cref{sec:appx_theory}.

\cutsubsectionup
\subsection{PMA Performance Bound}
\cutsubsectiondown
We first state the performance bound of PMA.
Formally, for the original MDP $\gM = (\gS, \gA, \mu, p, r)$,
we define the MDP with a \emph{learned} dynamics model $\hat{p}$
as $\hat{\gM} = (\gS, \gA, \mu, \hat{p}, r)$,
where we assume that $\hat{p}$ is trained on the dataset $\gD$ collected by $\pi_\gD$.
For our predictable latent MDP $\gM_P = (\gS, \gZ, \mu, p_z, r)$, we similarly define $\hat{\gM}_P = (\gS, \gZ, \mu, \hat{p}_z, r)$, $\gD_P$, and $\pi_{\gD_P}$.
For a policy $\pi(\va|\vs)$ in the original MDP, we define its corresponding latent policy
that best mimics the original one as $\pi_z^{\phi^*}(\vz|\vs)$,
and its induced next-state distribution as $p_z^{\phi^*}(\vs'|\vs, \va)$ (please see \Cref{sec:appx_theory} for the formal definitions).
We now state our performance bound of PMA with a learned latent dynamics model as follows:
\begin{theorem}
If the abstraction loss, the model error, and the policy difference are bounded as follows:
\begin{align}
    \E_{(\vs, \va) \sim d^{\pi}(\vs, \va)} [\TV (p(\cdot|\vs, \va) \| p_z^{\phi^*}(\cdot|\vs, \va))] &\leq \epsilon_a, \nonumber \\
    \E_{(\vs, \vz) \sim d^{\pi_{\gD_P}}(\vs, \vz)} [\TV (p_z(\cdot|\vs, \vz) \| \hat{p}_z(\cdot|\vs, \vz))] &\leq \epsilon_m', \nonumber \\
    \E_{\vs \sim d^{\pi_{\gD_P}}(\vs)} [\TV (\pi_z^{\phi^*}(\cdot|\vs) \| \pi_{\gD_P}(\cdot|\vs))] &\leq \epsilon_\pi', \nonumber
\end{align}
the performance difference of $\pi$ between the original MDP and the predictable latent model-based MDP can be bounded as:
\begin{align}
    |J_{\gM}(\pi) - J_{\hat{\gM}_P}(\pi_z^{\phi^*})| \leq \frac{R}{(1-\gamma)^2}(2 \gamma \epsilon_a + 4 \epsilon_\pi' + 2\gamma \epsilon_m').
    \label{eq:main_pma_bound}
\end{align}
\end{theorem}
Intuitively, PMA's performance bound consists of the following three factors:
(\emph{i}) the degree to which we lose from the lossy action decoder ($\epsilon_a$),
(\emph{ii}) the model error in the latent predictive model ($\epsilon_m'$),
and (\emph{iii}) the distributional shift between the data-collecting policy and the task policy ($\epsilon_\pi'$).
Hence, the bound becomes tighter if we have better state-action coverage and lower model errors,
which is precisely what PMA aims to achieve (\Cref{eq:pma_decomp}).

\begin{figure}[t!]
    \centering
    \includegraphics[width=\linewidth]{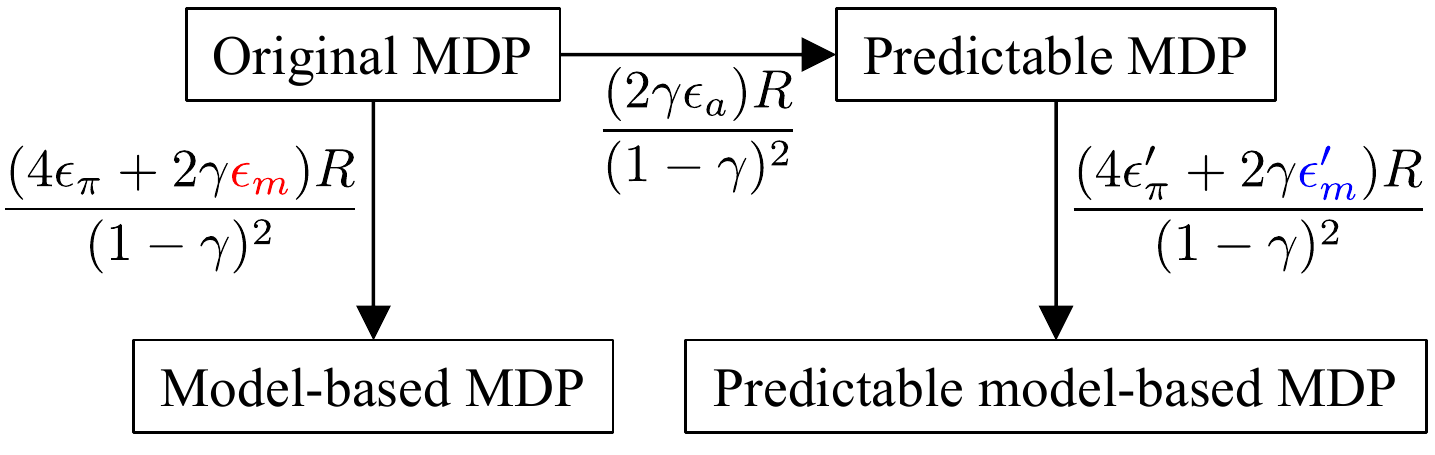}
    \caption{
    Performance bound between four MDPs.  %
    When ${\color{red} \epsilon_m} \gg \epsilon_a + {\color{blue} \epsilon_m'}$, PMA provides a tighter bound than classic MBRL.
    }
    \label{fig:pma_theory}
\end{figure}

\begin{figure*}[t!]
    \centering
    \includegraphics[width=\linewidth]{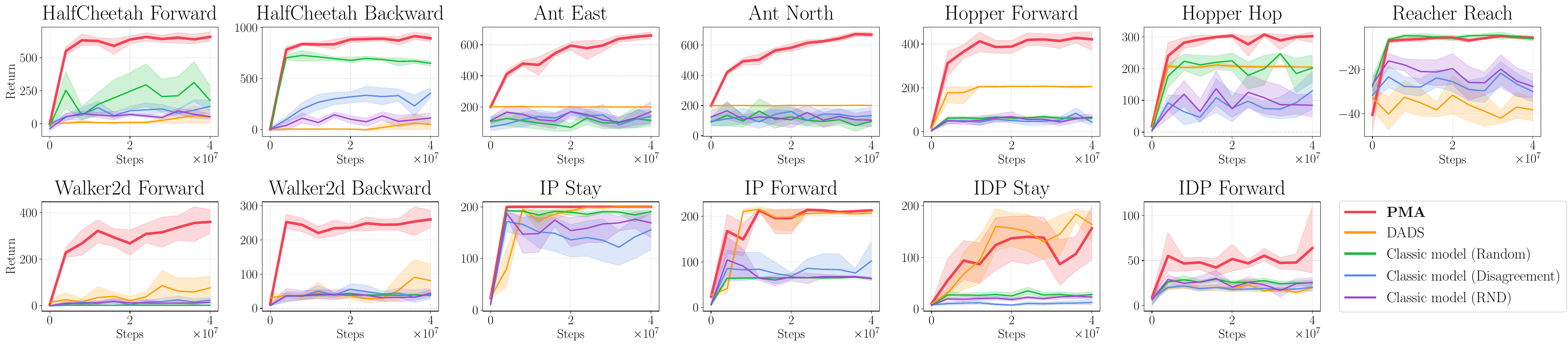}
    \vspace{-10pt}
    \caption{
    Comparison of periodic zero-shot planning (MPPI combined with the MOPO penalty)
    performances among unsupervised model-based RL methods.
    PMA demonstrates the best performance in most tasks, especially in Ant and Walker2d.
    }
    \vspace{-10pt}
    \label{fig:mppi_all}
\end{figure*}

\cutsubsectionup
\subsection{When Should We Use PMA over Classic MBRL?}
\cutsubsectiondown
\label{sec:when_pma}
To gain practical insights into PMA,
we theoretically compare PMA with classic model-based RL.
We first present the performance bound of classic MBRL \citep{mbpo_janner2019}:
\begin{theorem}
If the model error and the policy difference are bounded as follows:
\begin{align}
    \E_{(\vs, \va) \sim d^{\pi_\gD}(\vs, \va)} [\TV (p(\cdot|\vs, \va) \| \hat{p}(\cdot|\vs, \va))] &\leq \epsilon_m, \nonumber \\
    \E_{\vs \sim d^{\pi_\gD}(\vs)} [\TV (\pi(\cdot|\vs) \| \pi_\gD(\cdot|\vs))] &\leq \epsilon_\pi, \nonumber
\end{align}
the performance difference of $\pi$ between $\gM$ and $\hat{\gM}$ can be bounded as:
\begin{align}
    |J_{\gM}(\pi) - J_{\hat{\gM}}(\pi)| \leq \frac{R}{(1-\gamma)^2}(4 \epsilon_\pi + 2\gamma \epsilon_m). \label{eq:main_naive_bound}
\end{align}
\end{theorem}

By comparing \Cref{eq:main_pma_bound} and \Cref{eq:main_naive_bound},
we can see that PMA leads to a tighter bound when $4 \epsilon_\pi + 2\gamma {\color{red} \epsilon_m}
> 2 \gamma \epsilon_a + 4 \epsilon_\pi' + 2\gamma {\color{blue} \epsilon_m'}$ (\Cref{fig:pma_theory}).
Intuitively, this condition corresponds to ${\color{red} \epsilon_m} \gg \epsilon_a + {\color{blue} \epsilon_m'}$,
if we assume that the data collection policies in both cases have a similar divergence from the desired policy $\pi$.
This indicates that when the reduction in the model error by having predictable transitions outweighs the abstraction loss,
PMA can be more beneficial than classic MBRL.

\textbf{When can PMA be practically useful?}
PMA is useful when the optimal policies for the tasks %
mainly consist of predictable transitions so that we can reduce the model error $\epsilon_m'$ while maintaining a small $\epsilon_a$
(the average abstraction loss over the transition distribution of the optimal policies \emph{of our interest}).
For instance, in real-world driving scenarios,
we can achieve most of our driving purposes without modeling (and even by actively avoiding)
unpredictable behaviors like breaking the car in diverse ways,
which makes PMA beneficial.
Similar arguments can be applied to many robotics environments,
as we empirically demonstrate in \Cref{sec:exp}.
On the other hand, we could imagine MDPs
where optimal behavior requires intentionally visiting unpredictable regions of the state space,
in which case PMA could be suboptimal.

\cutsectionup
\section{Experiments}
\cutsectiondown
\label{sec:exp}

In our experiments, we study the performance of PMA as an unsupervised model-learning algorithm
for downstream \emph{zero-shot} model-based RL,
and analyze the degree to which PMA can learn more predictable models that enable longer-horizon simulated rollouts.
In particular, we aim to answer the following questions:
(\emph{i}) Does PMA lead to better zero-shot task performances in diverse tasks?
(\emph{ii}) Does PMA enable robust longer-horizon planning, without suffering from model exploitation?
(\emph{iii}) Does PMA learn more predictable models?

\textbf{Experimental setup.}
Since PMA does not require access to the task reward during the model training process,
we focus our comparisons on other \emph{unsupervised} model-based RL methods
that operate under similar assumptions:
a pre-training phase with interactive access to the MDP but not its reward function,
followed by a \emph{zero-shot} evaluation phase.
Previous unsupervised model-based approaches \citep{max_shyam2019,disag_pathak2019,p2e_sekar2020,rajeswar2022}
typically pre-train a classic dynamics model of the form $\hat{p}(\vs'|\vs, \va)$
using data gathered by some exploration policy.
We consider three of them as our baselines:
classic models (CMs) $\hat{p}(\vs'|\vs, \va)$ trained with
(\emph{i}) random actions (``Random''),
(\emph{ii}) ensemble disagreement-based exploration (``Disagreement'') \citep{disag_pathak2019},
which was previously proposed as an unsupervised data collection scheme for model learning
in several works~\citep{max_shyam2019,p2e_sekar2020,rajeswar2022},
and (\emph{iii}) random network distillation (``RND'') \citep{rnd_burda2019},
another data collection method considered by \citet{rajeswar2022}.
We also compare to DADS \citep{dads_sharma2020},
a previous unsupervised skill discovery method
that also learns a latent action dynamics model $\hat{p}_z(\vs'|\vs, \vz)$
but aims to find compact, temporally extended behaviors, rather than converting the original MDP into a more predictable one.
For the benchmark,
we test PMA and the four previous methods
on seven MuJoCo robotics environments \citep{mujoco_todorov2012,openaigym_brockman2016}
with $13$ diverse tasks.
We note that,
in our experiments,
we always use an ensemble disagreement penalty (MOPO penalty, \Cref{sec:mopo}) individually tuned for
every method, task, and controller, to ensure fair comparisons.
Every experiment is run with $8$ random seeds and we present $95\%$ confidence intervals in the plots.

\cutsubsectionup
\subsection{Model-Based Planning with PMA}
\cutsubsectiondown
In order to examine whether PMA leads to better planning performance,
we first evaluate the models learned by PMA
and each of the prior approaches using zero-shot planning for a downstream task.
PMA and DADS use latent models $\hat{p}_z(\vs'|\vs, \vz)$,
while the other methods all use ``classic'' models of the form $\hat{p}(\vs'|\vs, \va)$,
and differ only in their unsupervised data collection strategy.
We perform the comparison on seven MuJoCo environments (HalfCheetah, Ant, Hopper, Walker2d, InvertedPendulum (``IP''), InvertedDoublePendulum (``IDP''), and Reacher)
with $13$ tasks.
During unsupervised training of these methods,
we periodically run MPPI planning (\Cref{sec:mbrl_ex}) combined with the MOPO penalty on the downstream tasks
(these trials are not used for model training, which is completely unaware of the task reward),
and report its results in \Cref{fig:mppi_all}.
The results show that PMA achieves the best performance in most tasks.
Especially, PMA is the only successful unsupervised model-based method in Ant,
whose complex, contact-rich dynamics make it difficult for classic models to succeed
because erroneous model predictions often result in the agent flipping over.
PMA successfully solves the tasks
since our predictable abstraction effectively prevents such errors.

\begin{figure}[t!]
    \centering
    \includegraphics[width=\linewidth]{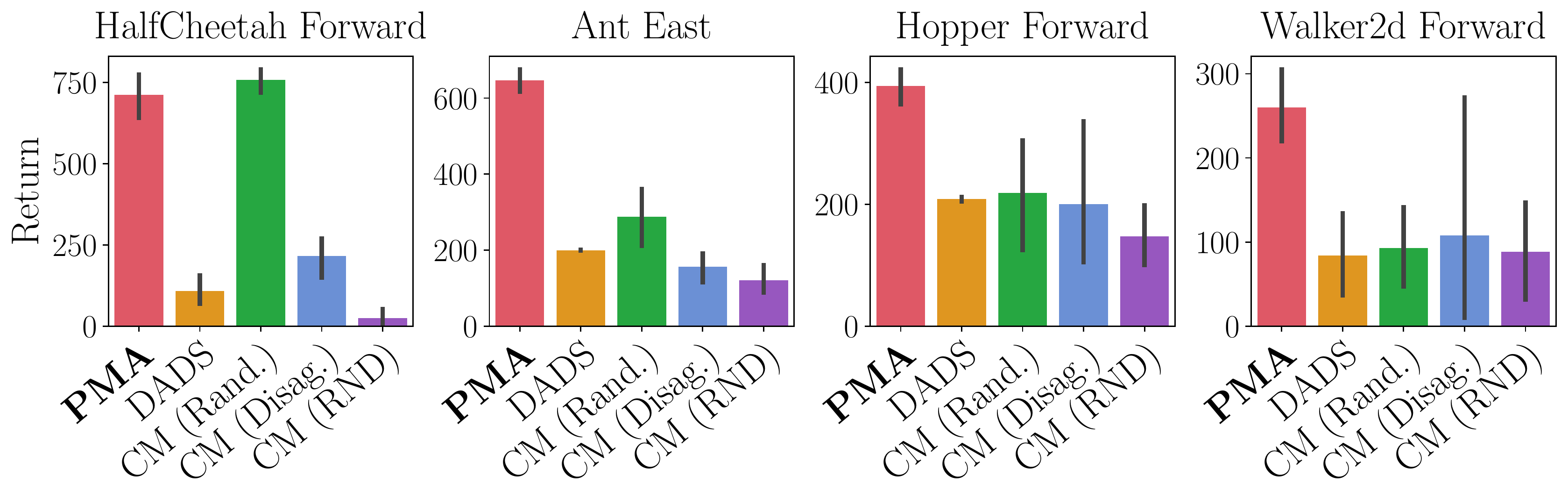}
    \vspace{-20pt}
    \caption{
    Comparison of unsupervised model-based RL methods using
    MBPO combined with the MOPO penalty.
    PMA mostly outperforms the prior approaches often by large margins.
    }
    \label{fig:mbpo_all}
\end{figure}
\begin{figure}[t!]
    \centering
    \includegraphics[width=\linewidth]{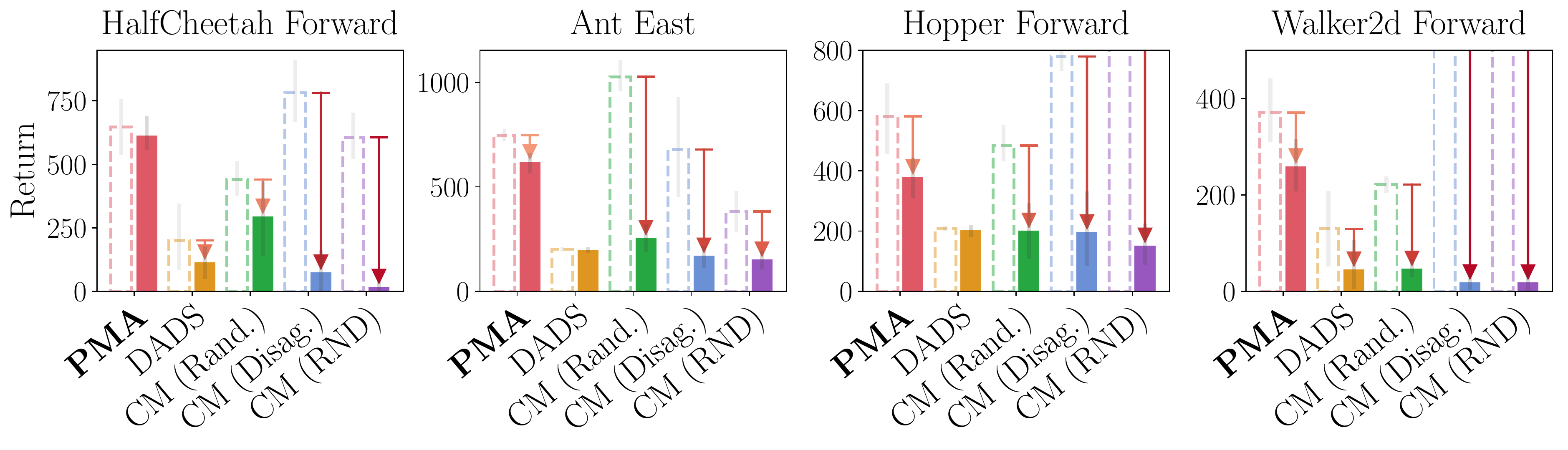}
    \vspace{-20pt}
    \caption{
    Performances of SAC combined with the MOPO penalty trained on purely model-based full-length rollouts.
    Dotted boxes indicate predicted returns and solid boxes indicate true returns.
    While classic models suffer from model exploitation in this long-horizon setting,
    as indicated by the drop from the predicted return to the actual return,
    PMA suffers a modest drop from the predicted return, and performs significantly better.
    }
    \vspace{-5pt}
    \label{fig:sac_all}
\end{figure}

\cutsubsectionup
\subsection{Model-Based RL with PMA}
\cutsubsectiondown
While the planning method used in the previous section uses the learned models with short-horizon MPC,
we can better evaluate the capacity of the PMA model and the baselines to make faithful long-horizon predictions
by using them with a long-horizon model-free RL procedure,
essentially treating the model as a ``simulator.''
We study two approaches in this section.
The first is based on MBPO \citep{mbpo_janner2019}, which we describe in \Cref{sec:mbrl_ex}.
The second approach is SAC \citep{sac_haarnoja2018}
on top of full-length model-based rollouts, %
which in some sense is the most literal way to use the learned model as a ``simulator'' of the true environment.
This second approach requires significantly longer model-based rollouts (up to $200$ time steps),
and though it performs worse in practice,
it provides a much more stringent test of the models' ability
to make long-horizon predictions without excessive error accumulation.
In both evaluation schemes, we use the MOPO penalty to prevent distributional shifts.

\Cref{fig:mbpo_all} and \Cref{fig:sac_all} present the results with MBPO and SAC, respectively.
In both settings, PMA mostly outperforms prior model-based approaches,
suggesting that PMA is capable of making reliable long-horizon predictions.
Also, by comparing the Hopper and Walker2d performances of \Cref{fig:mppi_all} and \Cref{fig:mbpo_all},
we find that classic models fail with MPPI and require a complex controller like MBPO to succeed,
whereas PMA with a simple planner can achieve similar results to MBPO owing to its predictable dynamics.
In the full-length SAC plots in \Cref{fig:sac_all},
we compare the models' predicted returns and the actual returns.
We find that the drops from PMA's predicted returns to actual returns are generally modest,
which indicates that our predictable abstraction effectively prevents model exploitation.
DADS similarly shows small performance differences as it also restricts actions,
but the absolute performance of DADS falls significantly behind PMA
due to its limited coverage of the state-action space.
On the other hand,
classic models tend to be erroneously optimistic about the returns because of the complex dynamics,
suffering from model exploitation.

\begin{figure}[t!]
    \centering
    \includegraphics[width=\linewidth]{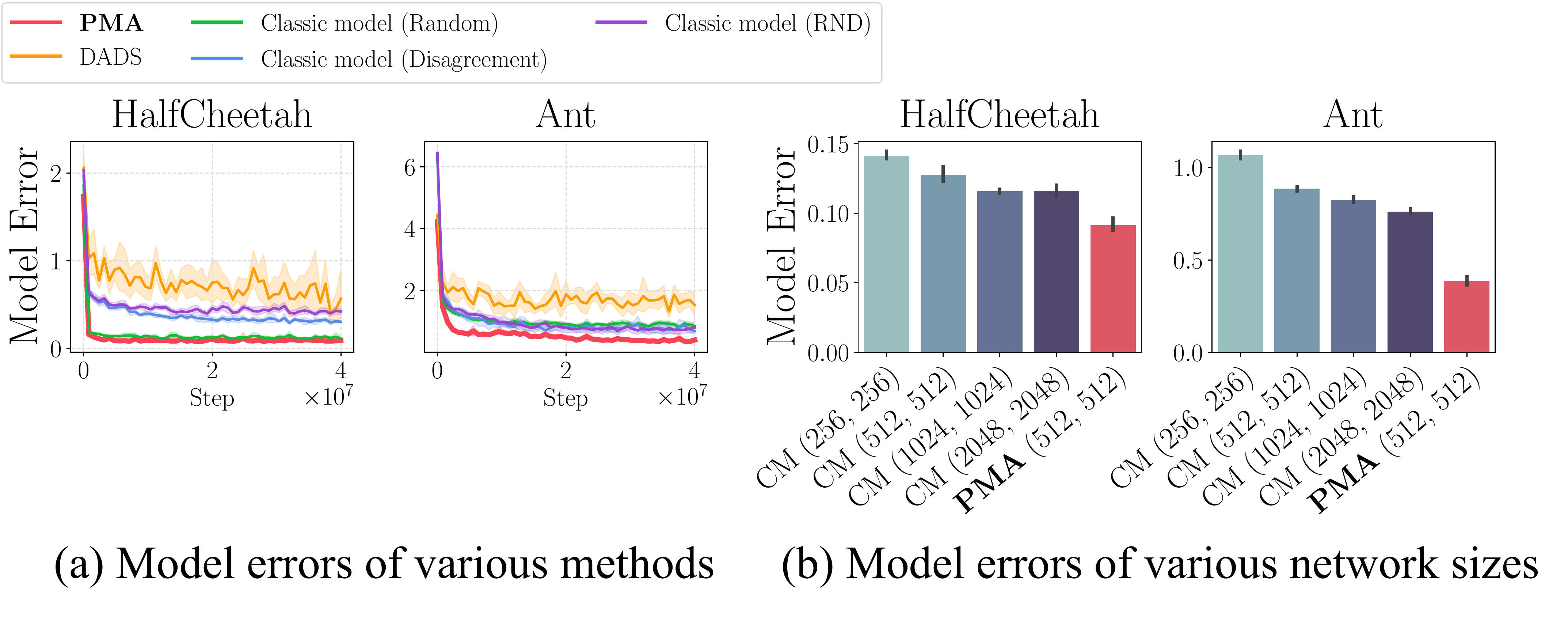}
    \vspace{-20pt}
    \caption{
    (a) PMA achieves the lowest model errors due to our predictability objective (\pmavideo).
    (b) Even in deterministic environments, it is challenging to completely reduce
    the errors in classic models even with larger neural networks.
    }
    \vspace{-10pt}
    \label{fig:mse}
\end{figure}
\begin{figure}[t!]
    \centering
    \includegraphics[width=\linewidth]{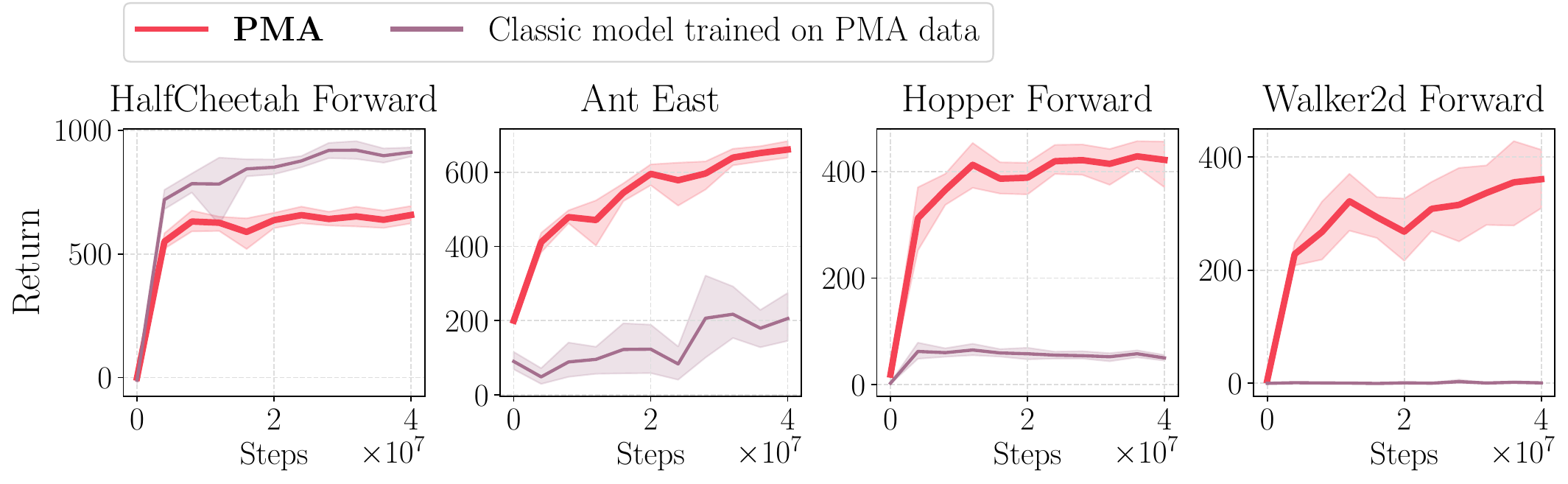}
    \vspace{-10pt}
    \caption{
    MPPI performance comparison between PMA and classic models trained with the data collected by PMA.
    While HalfCheetah does not benefit from action reparameterization,
    having a separate latent action space is crucial in the other environments.
    }
    \vspace{-10pt}
    \label{fig:am}
\end{figure}

\cutsubsectionup
\subsection{Analysis}
\cutsubsectiondown
\label{sec:analysis}

\textbf{Can model errors be reduced by simply increasing the model size?}
We first compare the average mean squared errors of predicted (normalized) states of the five methods
in HalfCheetah and Ant,
and report the results in \Cref{fig:mse}a.
In both environments, PMA exhibits the lowest model error, as it is trained to be maximally predictable.
To examine whether this degree of error can be achieved
by simply increasing the size of a classic model $\hat{p}(\vs'|\vs, \va)$,
we train classic models with random actions
using four different model sizes, ranging from two $256$-sized hidden layers to two $2048$-sized ones.
\Cref{fig:mse}b shows the results,
which suggest that there are virtually irreducible model errors
even in deterministic environments due to their complex, contact-rich dynamics.
On the other hand, PMA seeks to model not the entire MDP but only the ``predictable'' parts of the action space,
which reduces model errors and thus makes it amenable to model-based learning.

\textbf{Data restriction vs. action reparameterization.}
PMA serves both (\emph{i}) data restriction by not selecting unpredictable actions %
and (\emph{ii}) action reparameterization by having a separate latent action space.
To dissect these two effects, we additionally consider a classic model $\hat{p}(\vs'|\vs, \va)$
trained with the same data used to train PMA.
We compare the periodic MPPI performances of this setting and PMA in \Cref{fig:am}.
The results suggest that
while data restriction, in combination with the MOPO penalty, is sufficient in HalfCheetah,
having a separate latent action space is crucial in the other ``unstable'' environments
with early termination or irreversible states (\eg, flipping over in Ant).
This is because while the MOPO penalty at test time only prevents short-term deviations from the data distribution,
PMA trained with \emph{RL} considers long-term predictability,
which effectively prevents selecting seemingly benign actions that could eventually cause the agent to lose balance
(which corresponds to unpredictable behavior).

We refer to \pmalink and \Cref{sec:appx_qual} for qualitative results %
and \Cref{sec:appx_abl} for an ablation study. %

\cutsectionup
\section{Conclusion}
\cutsectiondown
We presented predictable MDP abstraction (PMA) as an unsupervised model-based method
that builds a latent MDP by reparameterizing the action space to only allow predictable actions.
We formulated its objective with information theory
and theoretically analyzed the suboptimality induced by the lossy training scheme.
Empirically, we confirmed that PMA enables robust model-based learning,
exhibiting significant performance improvements over prior approaches.

\textbf{Limitations.}
One limitation of PMA is that it is a lossy procedure.
While we empirically demonstrated that its improved predictability outweighs
the limitations imposed by the restriction of the action space in our experiments,
PMA might be suboptimal in tasks that require unpredictable or highly complex behaviors,
such those as discussed in \Cref{sec:when_pma},
and in general it may be difficult to guarantee that the abstractions learned by PMA are good
for every downstream task (though such guarantees are likely difficult to provide for any method).
Also, PMA requires tuning the coefficient $\beta$ to maintain a balance between predictability and the state-action space coverage. 
Nonetheless, we believe that methods that aim to specifically model the \emph{predictable} parts of an MDP
hold a lot of promise for future model-based RL methods.
Future work could explore hybridizing such techniques with model-free approaches
for handling the ``unpredictable'' parts,
further study effective data collection strategies
or methods that can utilize previously collected offline data,
and examine how such approaches could be scaled up to more complex and high-dimensional observation spaces,
where training directly for predictability could lead to even more significant gains in performance.

\section*{Acknowledgement}
We would like to thank Benjamin Eysenbach for insightful discussions about the initial idea,
Mitsuhiko Nakamoto, Jaekyeom Kim, and Youngwoon Lee for discussions about implementation and presentation,
and RAIL members and anonymous reviewers for their helpful comments.
Seohong Park was partly supported by Korea Foundation for Advanced Studies (KFAS).
This research used the Savio computational cluster resource provided by the Berkeley Research Computing program at UC Berkeley
and was partly supported by AFOSR FA9550-22-1-0273 and the Office of Naval Research.

\bibliography{pma}

\begin{thebibliography}{80}
\providecommand{\natexlab}[1]{#1}
\providecommand{\url}[1]{\texttt{#1}}
\expandafter\ifx\csname urlstyle\endcsname\relax
  \providecommand{\doi}[1]{doi: #1}\else
  \providecommand{\doi}{doi: \begingroup \urlstyle{rm}\Url}\fi

\bibitem[Ajay et~al.(2021)Ajay, Kumar, Agrawal, Levine, and
  Nachum]{opal_ajay2021}
Ajay, A., Kumar, A., Agrawal, P., Levine, S., and Nachum, O.
\newblock Opal: Offline primitive discovery for accelerating offline
  reinforcement learning.
\newblock In \emph{International Conference on Learning Representations
  (ICLR)}, 2021.

\bibitem[Argenson \& Dulac-Arnold(2021)Argenson and
  Dulac-Arnold]{mbop_argenson2021}
Argenson, A. and Dulac-Arnold, G.
\newblock Model-based offline planning.
\newblock In \emph{International Conference on Learning Representations
  (ICLR)}, 2021.

\bibitem[Bacon et~al.(2017)Bacon, Harb, and Precup]{oc_bacon2017}
Bacon, P.-L., Harb, J., and Precup, D.
\newblock The option-critic architecture.
\newblock In \emph{AAAI Conference on Artificial Intelligence (AAAI)}, 2017.

\bibitem[Ball et~al.(2020)Ball, Parker-Holder, Pacchiano, Choromanski, and
  Roberts]{rp1_ball2020}
Ball, P.~J., Parker-Holder, J., Pacchiano, A., Choromanski, K., and Roberts,
  S.~J.
\newblock Ready policy one: World building through active learning.
\newblock In \emph{International Conference on Machine Learning (ICML)}, 2020.

\bibitem[Barber \& Agakov(2003)Barber and Agakov]{im_barber2003}
Barber, D. and Agakov, F.
\newblock The {IM} algorithm: a variational approach to information
  maximization.
\newblock In \emph{Neural Information Processing Systems (NeurIPS)}, 2003.

\bibitem[Berseth et~al.(2021)Berseth, Geng, Devin, Rhinehart, Finn, Jayaraman,
  and Levine]{smirl_berseth2021}
Berseth, G., Geng, D., Devin, C., Rhinehart, N., Finn, C., Jayaraman, D., and
  Levine, S.
\newblock Smirl: Surprise minimizing reinforcement learning in unstable
  environments.
\newblock In \emph{International Conference on Learning Representations
  (ICLR)}, 2021.

\bibitem[Brockman et~al.(2016)Brockman, Cheung, Pettersson, Schneider,
  Schulman, Tang, and Zaremba]{openaigym_brockman2016}
Brockman, G., Cheung, V., Pettersson, L., Schneider, J., Schulman, J., Tang,
  J., and Zaremba, W.
\newblock {OpenAI Gym}.
\newblock \emph{ArXiv}, abs/1606.01540, 2016.

\bibitem[Buckman et~al.(2018)Buckman, Hafner, Tucker, Brevdo, and
  Lee]{steve_buckman2018}
Buckman, J., Hafner, D., Tucker, G., Brevdo, E., and Lee, H.
\newblock Sample-efficient reinforcement learning with stochastic ensemble
  value expansion.
\newblock In \emph{Neural Information Processing Systems (NeurIPS)}, 2018.

\bibitem[Burda et~al.(2019)Burda, Edwards, Storkey, and Klimov]{rnd_burda2019}
Burda, Y., Edwards, H., Storkey, A.~J., and Klimov, O.
\newblock Exploration by random network distillation.
\newblock In \emph{International Conference on Learning Representations
  (ICLR)}, 2019.

\bibitem[Campos~Cam{\'u}{\~n}ez et~al.(2020)Campos~Cam{\'u}{\~n}ez, Trott,
  Xiong, Socher, Gir{\'o}~Nieto, and Torres~Vi{\~n}als]{edl_campos2020}
Campos~Cam{\'u}{\~n}ez, V., Trott, A., Xiong, C., Socher, R., Gir{\'o}~Nieto,
  X., and Torres~Vi{\~n}als, J.
\newblock Explore, discover and learn: unsupervised discovery of state-covering
  skills.
\newblock In \emph{International Conference on Machine Learning (ICML)}, 2020.

\bibitem[Castro(2019)]{bisim_castro2019}
Castro, P.~S.
\newblock Scalable methods for computing state similarity in deterministic
  markov decision processes.
\newblock In \emph{AAAI Conference on Artificial Intelligence (AAAI)}, 2019.

\bibitem[Chua et~al.(2018)Chua, Calandra, McAllister, and
  Levine]{pets_chua2018}
Chua, K., Calandra, R., McAllister, R., and Levine, S.
\newblock Deep reinforcement learning in a handful of trials using
  probabilistic dynamics models.
\newblock In \emph{Neural Information Processing Systems (NeurIPS)}, 2018.

\bibitem[Clavera et~al.(2018)Clavera, Rothfuss, Schulman, Fujita, Asfour, and
  Abbeel]{mbmpo_clavera2018}
Clavera, I., Rothfuss, J., Schulman, J., Fujita, Y., Asfour, T., and Abbeel, P.
\newblock Model-based reinforcement learning via meta-policy optimization.
\newblock In \emph{Conference on Robot Learning (CoRL)}, 2018.

\bibitem[Deisenroth \& Rasmussen(2011)Deisenroth and
  Rasmussen]{pilco_deisenroth2011}
Deisenroth, M.~P. and Rasmussen, C.~E.
\newblock Pilco: A model-based and data-efficient approach to policy search.
\newblock In \emph{International Conference on Machine Learning (ICML)}, 2011.

\bibitem[Draeger et~al.(1995)Draeger, Engell, and Ranke]{draeger1995}
Draeger, A., Engell, S., and Ranke, H.~D.
\newblock Model predictive control using neural networks.
\newblock \emph{IEEE Control Systems Magazine}, 15:\penalty0 61--66, 1995.

\bibitem[Ebert et~al.(2018)Ebert, Finn, Dasari, Xie, Lee, and
  Levine]{visualmpc_ebert2018}
Ebert, F., Finn, C., Dasari, S., Xie, A., Lee, A.~X., and Levine, S.
\newblock Visual foresight: Model-based deep reinforcement learning for
  vision-based robotic control.
\newblock \emph{ArXiv}, abs/1812.00568, 2018.

\bibitem[Eysenbach et~al.(2019)Eysenbach, Gupta, Ibarz, and
  Levine]{diayn_eysenbach2019}
Eysenbach, B., Gupta, A., Ibarz, J., and Levine, S.
\newblock Diversity is all you need: Learning skills without a reward function.
\newblock In \emph{International Conference on Learning Representations
  (ICLR)}, 2019.

\bibitem[Eysenbach et~al.(2021)Eysenbach, Salakhutdinov, and
  Levine]{rpc_eysenbach2021}
Eysenbach, B., Salakhutdinov, R., and Levine, S.
\newblock Robust predictable control.
\newblock In \emph{Neural Information Processing Systems (NeurIPS)}, 2021.

\bibitem[Feinberg et~al.(2018)Feinberg, Wan, Stoica, Jordan, Gonzalez, and
  Levine]{mve_feinberg2018}
Feinberg, V., Wan, A., Stoica, I., Jordan, M.~I., Gonzalez, J.~E., and Levine,
  S.
\newblock Model-based value expansion for efficient model-free reinforcement
  learning.
\newblock In \emph{International Conference on Machine Learning (ICML)}, 2018.

\bibitem[Gelada et~al.(2019)Gelada, Kumar, Buckman, Nachum, and
  Bellemare]{deepmdp_gelada2019}
Gelada, C., Kumar, S., Buckman, J., Nachum, O., and Bellemare, M.~G.
\newblock Deepmdp: Learning continuous latent space models for representation
  learning.
\newblock In \emph{International Conference on Machine Learning (ICML)}, 2019.

\bibitem[Gregor et~al.(2016)Gregor, Rezende, and Wierstra]{vic_gregor2016}
Gregor, K., Rezende, D.~J., and Wierstra, D.
\newblock Variational intrinsic control.
\newblock \emph{ArXiv}, abs/1611.07507, 2016.

\bibitem[Ha \& Schmidhuber(2018)Ha and Schmidhuber]{world_ha2018}
Ha, D. and Schmidhuber, J.
\newblock Recurrent world models facilitate policy evolution.
\newblock In \emph{Neural Information Processing Systems (NeurIPS)}, 2018.

\bibitem[Haarnoja et~al.(2018{\natexlab{a}})Haarnoja, Zhou, Abbeel, and
  Levine]{sac_haarnoja2018}
Haarnoja, T., Zhou, A., Abbeel, P., and Levine, S.
\newblock Soft actor-critic: Off-policy maximum entropy deep reinforcement
  learning with a stochastic actor.
\newblock In \emph{International Conference on Machine Learning (ICML)},
  2018{\natexlab{a}}.

\bibitem[Haarnoja et~al.(2018{\natexlab{b}})Haarnoja, Zhou, Hartikainen,
  Tucker, Ha, Tan, Kumar, Zhu, Gupta, Abbeel, and Levine]{saces_haarnoja2018}
Haarnoja, T., Zhou, A., Hartikainen, K., Tucker, G., Ha, S., Tan, J., Kumar,
  V., Zhu, H., Gupta, A., Abbeel, P., and Levine, S.
\newblock Soft actor-critic algorithms and applications.
\newblock \emph{ArXiv}, abs/1812.05905, 2018{\natexlab{b}}.

\bibitem[Hafner et~al.(2019)Hafner, Lillicrap, Fischer, Villegas, Ha, Lee, and
  Davidson]{planet_hafner2019}
Hafner, D., Lillicrap, T.~P., Fischer, I.~S., Villegas, R., Ha, D.~R., Lee, H.,
  and Davidson, J.
\newblock Learning latent dynamics for planning from pixels.
\newblock In \emph{International Conference on Machine Learning (ICML)}, 2019.

\bibitem[Hafner et~al.(2020)Hafner, Lillicrap, Ba, and
  Norouzi]{dreamer_hafner2020}
Hafner, D., Lillicrap, T.~P., Ba, J., and Norouzi, M.
\newblock Dream to control: Learning behaviors by latent imagination.
\newblock In \emph{International Conference on Learning Representations
  (ICLR)}, 2020.

\bibitem[Hansen et~al.(2022)Hansen, Wang, and Su]{tdmpc_hansen2022}
Hansen, N., Wang, X., and Su, H.
\newblock Temporal difference learning for model predictive control.
\newblock In \emph{International Conference on Machine Learning (ICML)}, 2022.

\bibitem[Hasselt et~al.(2019)Hasselt, Hessel, and Aslanides]{hasselt2019}
Hasselt, H.~V., Hessel, M., and Aslanides, J.
\newblock When to use parametric models in reinforcement learning?
\newblock In \emph{Neural Information Processing Systems (NeurIPS)}, 2019.

\bibitem[Heess et~al.(2015)Heess, Wayne, Silver, Lillicrap, Erez, and
  Tassa]{svg_heess2015}
Heess, N. M.~O., Wayne, G., Silver, D., Lillicrap, T.~P., Erez, T., and Tassa,
  Y.
\newblock Learning continuous control policies by stochastic value gradients.
\newblock In \emph{Neural Information Processing Systems (NeurIPS)}, 2015.

\bibitem[Hernandaz \& Arkun(1990)Hernandaz and Arkun]{hernandaz1990}
Hernandaz, E. P.~S. and Arkun, Y.
\newblock Neural network modeling and an extended dmc algorithm to control
  nonlinear systems.
\newblock In \emph{American Control Conference}, 1990.

\bibitem[Houthooft et~al.(2016)Houthooft, Chen, Duan, Schulman, Turck, and
  Abbeel]{vime_houthooft2016}
Houthooft, R., Chen, X., Duan, Y., Schulman, J., Turck, F.~D., and Abbeel, P.
\newblock Vime: Variational information maximizing exploration.
\newblock In \emph{Neural Information Processing Systems (NeurIPS)}, 2016.

\bibitem[Jafferjee et~al.(2020)Jafferjee, Imani, Talvitie, White, and
  Bowling]{hvh_jafferjee2020}
Jafferjee, T., Imani, E., Talvitie, E.~J., White, M., and Bowling, M.
\newblock Hallucinating value: A pitfall of dyna-style planning with imperfect
  environment models.
\newblock \emph{ArXiv}, abs/2006.04363, 2020.

\bibitem[Janner et~al.(2019)Janner, Fu, Zhang, and Levine]{mbpo_janner2019}
Janner, M., Fu, J., Zhang, M., and Levine, S.
\newblock When to trust your model: Model-based policy optimization.
\newblock In \emph{Neural Information Processing Systems (NeurIPS)}, 2019.

\bibitem[Kang et~al.(2022)Kang, Gradu, Choi, Janner, Tomlin, and
  Levine]{ldm_kang2022}
Kang, K., Gradu, P., Choi, J.~J., Janner, M., Tomlin, C.~J., and Levine, S.
\newblock Lyapunov density models: Constraining distribution shift in
  learning-based control.
\newblock In \emph{International Conference on Machine Learning (ICML)}, 2022.

\bibitem[Kidambi et~al.(2020)Kidambi, Rajeswaran, Netrapalli, and
  Joachims]{morel_kidambi2020}
Kidambi, R., Rajeswaran, A., Netrapalli, P., and Joachims, T.
\newblock Morel : Model-based offline reinforcement learning.
\newblock In \emph{Neural Information Processing Systems (NeurIPS)}, 2020.

\bibitem[Kim et~al.(2021)Kim, Park, and Kim]{ibol_kim2021}
Kim, J., Park, S., and Kim, G.
\newblock Unsupervised skill discovery with bottleneck option learning.
\newblock In \emph{International Conference on Machine Learning (ICML)}, 2021.

\bibitem[Kingma \& Ba(2015)Kingma and Ba]{adam_kingma2014}
Kingma, D.~P. and Ba, J.
\newblock Adam: A method for stochastic optimization.
\newblock In \emph{International Conference on Learning Representations
  (ICLR)}, 2015.

\bibitem[Kurutach et~al.(2018)Kurutach, Clavera, Duan, Tamar, and
  Abbeel]{metrpo_kurutach2018}
Kurutach, T., Clavera, I., Duan, Y., Tamar, A., and Abbeel, P.
\newblock Model-ensemble trust-region policy optimization.
\newblock In \emph{International Conference on Learning Representations
  (ICLR)}, 2018.

\bibitem[Laskin et~al.(2022)Laskin, Liu, Peng, Yarats, Rajeswaran, and
  Abbeel]{cic_laskin2022}
Laskin, M., Liu, H., Peng, X.~B., Yarats, D., Rajeswaran, A., and Abbeel, P.
\newblock Unsupervised reinforcement learning with contrastive intrinsic
  control.
\newblock In \emph{Neural Information Processing Systems (NeurIPS)}, 2022.

\bibitem[Lee et~al.(2019)Lee, Eysenbach, Parisotto, Xing, Levine, and
  Salakhutdinov]{smm_lee2019}
Lee, L., Eysenbach, B., Parisotto, E., Xing, E.~P., Levine, S., and
  Salakhutdinov, R.
\newblock Efficient exploration via state marginal matching.
\newblock \emph{ArXiv}, abs/1906.05274, 2019.

\bibitem[Lenz et~al.(2015)Lenz, Knepper, and Saxena]{deepmpc_lenz2015}
Lenz, I., Knepper, R.~A., and Saxena, A.
\newblock Deepmpc: Learning deep latent features for model predictive control.
\newblock In \emph{Robotics: Science and Systems (RSS)}, 2015.

\bibitem[Li et~al.(2006)Li, Walsh, and Littman]{bisim_li2006}
Li, L., Walsh, T.~J., and Littman, M.~L.
\newblock Towards a unified theory of state abstraction for mdps.
\newblock In \emph{International Symposium on Artificial Intelligence and
  Mathematics}, 2006.

\bibitem[Liu \& Abbeel(2021)Liu and Abbeel]{apt_liu2021}
Liu, H. and Abbeel, P.
\newblock Behavior from the void: Unsupervised active pre-training.
\newblock In \emph{Neural Information Processing Systems (NeurIPS)}, 2021.

\bibitem[Lu et~al.(2021)Lu, Grover, Abbeel, and Mordatch]{lisp_lu2021}
Lu, K., Grover, A., Abbeel, P., and Mordatch, I.
\newblock Reset-free lifelong learning with skill-space planning.
\newblock In \emph{International Conference on Learning Representations
  (ICLR)}, 2021.

\bibitem[Machado et~al.(2017)Machado, Bellemare, and
  Bowling]{eigen_machado2017}
Machado, M.~C., Bellemare, M.~G., and Bowling, M.
\newblock A laplacian framework for option discovery in reinforcement learning.
\newblock In \emph{International Conference on Machine Learning (ICML)}, 2017.

\bibitem[Mendonca et~al.(2021)Mendonca, Rybkin, Daniilidis, Hafner, and
  Pathak]{lexa_mendonca2021}
Mendonca, R., Rybkin, O., Daniilidis, K., Hafner, D., and Pathak, D.
\newblock Discovering and achieving goals via world models.
\newblock In \emph{Neural Information Processing Systems (NeurIPS)}, 2021.

\bibitem[Nachum et~al.(2018)Nachum, Gu, Lee, and Levine]{hiro_nachum2018}
Nachum, O., Gu, S.~S., Lee, H., and Levine, S.
\newblock Data-efficient hierarchical reinforcement learning.
\newblock In \emph{Neural Information Processing Systems (NeurIPS)}, 2018.

\bibitem[Nachum et~al.(2019)Nachum, Gu, Lee, and
  Levine]{nearoptimal_nachum2019}
Nachum, O., Gu, S.~S., Lee, H., and Levine, S.
\newblock Near-optimal representation learning for hierarchical reinforcement
  learning.
\newblock In \emph{International Conference on Learning Representations
  (ICLR)}, 2019.

\bibitem[Nagabandi et~al.(2019)Nagabandi, Konolige, Levine, and
  Kumar]{pddm_nagabandi2019}
Nagabandi, A., Konolige, K., Levine, S., and Kumar, V.
\newblock Deep dynamics models for learning dexterous manipulation.
\newblock In \emph{Conference on Robot Learning (CoRL)}, 2019.

\bibitem[Nguyen et~al.(2021)Nguyen, Shu, Pham, Bui, and Ermon]{tpc_nguyen2021}
Nguyen, T.~D., Shu, R., Pham, T., Bui, H.~H., and Ermon, S.
\newblock Temporal predictive coding for model-based planning in latent space.
\newblock In \emph{International Conference on Machine Learning (ICML)}, 2021.

\bibitem[Park et~al.(2022)Park, Choi, Kim, Lee, and Kim]{lsd_park2022}
Park, S., Choi, J., Kim, J., Lee, H., and Kim, G.
\newblock Lipschitz-constrained unsupervised skill discovery.
\newblock In \emph{International Conference on Learning Representations
  (ICLR)}, 2022.

\bibitem[Pathak et~al.(2017)Pathak, Agrawal, Efros, and
  Darrell]{icm_pathak2017}
Pathak, D., Agrawal, P., Efros, A.~A., and Darrell, T.
\newblock Curiosity-driven exploration by self-supervised prediction.
\newblock In \emph{International Conference on Machine Learning (ICML)}, 2017.

\bibitem[Pathak et~al.(2019)Pathak, Gandhi, and Gupta]{disag_pathak2019}
Pathak, D., Gandhi, D., and Gupta, A.~K.
\newblock Self-supervised exploration via disagreement.
\newblock In \emph{International Conference on Machine Learning (ICML)}, 2019.

\bibitem[Pong et~al.(2020)Pong, Dalal, Lin, Nair, Bahl, and
  Levine]{skewfit_pong2020}
Pong, V.~H., Dalal, M., Lin, S., Nair, A., Bahl, S., and Levine, S.
\newblock {Skew-Fit}: State-covering self-supervised reinforcement learning.
\newblock In \emph{International Conference on Machine Learning (ICML)}, 2020.

\bibitem[Poole et~al.(2019)Poole, Ozair, van~den Oord, Alemi, and
  Tucker]{mi_poole2019}
Poole, B., Ozair, S., van~den Oord, A., Alemi, A.~A., and Tucker, G.
\newblock On variational bounds of mutual information.
\newblock In \emph{International Conference on Machine Learning (ICML)}, 2019.

\bibitem[Rajeswar et~al.(2022)Rajeswar, Mazzaglia, Verbelen, Pich'e, Dhoedt,
  Courville, and Lacoste]{rajeswar2022}
Rajeswar, S., Mazzaglia, P., Verbelen, T., Pich'e, A., Dhoedt, B., Courville,
  A.~C., and Lacoste, A.
\newblock Unsupervised model-based pre-training for data-efficient control from
  pixels.
\newblock \emph{ArXiv}, abs/2209.12016, 2022.

\bibitem[Rajeswaran et~al.(2017)Rajeswaran, Ghotra, Levine, and
  Ravindran]{epopt_rajeswaran2017}
Rajeswaran, A., Ghotra, S., Levine, S., and Ravindran, B.
\newblock Epopt: Learning robust neural network policies using model ensembles.
\newblock In \emph{International Conference on Learning Representations
  (ICLR)}, 2017.

\bibitem[Rasmussen \& Kuss(2003)Rasmussen and Kuss]{gp_rasmussen2003}
Rasmussen, C.~E. and Kuss, M.
\newblock Gaussian processes in reinforcement learning.
\newblock In \emph{Neural Information Processing Systems (NeurIPS)}, 2003.

\bibitem[Rhinehart et~al.(2021)Rhinehart, Wang, Berseth, Co-Reyes, Hafner,
  Finn, and Levine]{ic2_rhinehart2021}
Rhinehart, N., Wang, J., Berseth, G., Co-Reyes, J.~D., Hafner, D., Finn, C.,
  and Levine, S.
\newblock Information is power: Intrinsic control via information capture.
\newblock In \emph{Neural Information Processing Systems (NeurIPS)}, 2021.

\bibitem[Ross \& Bagnell(2012)Ross and Bagnell]{agnostic_ross2012}
Ross, S. and Bagnell, J.~A.
\newblock Agnostic system identification for model-based reinforcement
  learning.
\newblock In \emph{International Conference on Machine Learning (ICML)}, 2012.

\bibitem[Salter et~al.(2022)Salter, Wulfmeier, Tirumala, Heess, Riedmiller,
  Hadsell, and Rao]{mo2_salter2022}
Salter, S., Wulfmeier, M., Tirumala, D., Heess, N. M.~O., Riedmiller, M.~A.,
  Hadsell, R., and Rao, D.
\newblock Mo2: Model-based offline options.
\newblock In \emph{Conference on Lifelong Learning Agents (CoLLAs)}, 2022.

\bibitem[Schrittwieser et~al.(2020)Schrittwieser, Antonoglou, Hubert, Simonyan,
  Sifre, Schmitt, Guez, Lockhart, Hassabis, Graepel, Lillicrap, and
  Silver]{muzero_schrittwieser2020}
Schrittwieser, J., Antonoglou, I., Hubert, T., Simonyan, K., Sifre, L.,
  Schmitt, S., Guez, A., Lockhart, E., Hassabis, D., Graepel, T., Lillicrap,
  T., and Silver, D.
\newblock Mastering {Atari, Go, Chess and Shogi} by planning with a learned
  model.
\newblock \emph{Nature}, 588 7839:\penalty0 604--609, 2020.

\bibitem[Sekar et~al.(2020)Sekar, Rybkin, Daniilidis, Abbeel, Hafner, and
  Pathak]{p2e_sekar2020}
Sekar, R., Rybkin, O., Daniilidis, K., Abbeel, P., Hafner, D., and Pathak, D.
\newblock Planning to explore via self-supervised world models.
\newblock In \emph{International Conference on Machine Learning (ICML)}, 2020.

\bibitem[Sharma et~al.(2020)Sharma, Gu, Levine, Kumar, and
  Hausman]{dads_sharma2020}
Sharma, A., Gu, S., Levine, S., Kumar, V., and Hausman, K.
\newblock Dynamics-aware unsupervised discovery of skills.
\newblock In \emph{International Conference on Learning Representations
  (ICLR)}, 2020.

\bibitem[Shyam et~al.(2019)Shyam, Jaśkowski, and Gomez]{max_shyam2019}
Shyam, P., Jaśkowski, W., and Gomez, F.~J.
\newblock Model-based active exploration.
\newblock In \emph{International Conference on Machine Learning (ICML)}, 2019.

\bibitem[Sikchi et~al.(2022)Sikchi, Zhou, and Held]{loop_sikchi2022}
Sikchi, H.~S., Zhou, W., and Held, D.
\newblock Learning off-policy with online planning.
\newblock In \emph{Conference on Robot Learning (CoRL)}, 2022.

\bibitem[Stolle \& Precup(2002)Stolle and Precup]{option_stolle2002}
Stolle, M. and Precup, D.
\newblock Learning options in reinforcement learning.
\newblock In \emph{Symposium on Abstraction, Reformulation and Approximation},
  2002.

\bibitem[Strouse et~al.(2022)Strouse, Baumli, Warde-Farley, Mnih, and
  Hansen]{disdain_strouse2022}
Strouse, D., Baumli, K., Warde-Farley, D., Mnih, V., and Hansen, S.~S.
\newblock Learning more skills through optimistic exploration.
\newblock In \emph{International Conference on Learning Representations
  (ICLR)}, 2022.

\bibitem[Sutton(1991)]{dyna_sutton1990}
Sutton, R.~S.
\newblock Dyna, an integrated architecture for learning, planning, and
  reacting.
\newblock \emph{ACM Sigart Bulletin}, 2\penalty0 (4):\penalty0 160--163, 1991.

\bibitem[Sutton et~al.(1999)Sutton, Precup, and Singh]{option_sutton1999}
Sutton, R.~S., Precup, D., and Singh, S.
\newblock Between mdps and semi-mdps: A framework for temporal abstraction in
  reinforcement learning.
\newblock \emph{Artificial intelligence}, 112\penalty0 (1-2):\penalty0
  181--211, 1999.

\bibitem[Testud et~al.(1978)Testud, Richalet, Rault, and
  Papon]{mpc_richalet1978}
Testud, J., Richalet, J., Rault, A., and Papon, J.
\newblock Model predictive heuristic control: Applications to industial
  processes.
\newblock \emph{Automatica}, 14\penalty0 (5):\penalty0 413--428, 1978.

\bibitem[Todorov et~al.(2012)Todorov, Erez, and Tassa]{mujoco_todorov2012}
Todorov, E., Erez, T., and Tassa, Y.
\newblock Mujoco: A physics engine for model-based control.
\newblock In \emph{IEEE/RSJ International Conference on Intelligent Robots and
  Systems (IROS)}, 2012.

\bibitem[Vezhnevets et~al.(2017)Vezhnevets, Osindero, Schaul, Heess, Jaderberg,
  Silver, and Kavukcuoglu]{fun_vezhnevets2017}
Vezhnevets, A.~S., Osindero, S., Schaul, T., Heess, N. M.~O., Jaderberg, M.,
  Silver, D., and Kavukcuoglu, K.
\newblock Feudal networks for hierarchical reinforcement learning.
\newblock In \emph{International Conference on Machine Learning (ICML)}, 2017.

\bibitem[Watter et~al.(2015)Watter, Springenberg, Boedecker, and
  Riedmiller]{e2c_watter2015}
Watter, M., Springenberg, J.~T., Boedecker, J., and Riedmiller, M.~A.
\newblock Embed to control: A locally linear latent dynamics model for control
  from raw images.
\newblock In \emph{Neural Information Processing Systems (NeurIPS)}, 2015.

\bibitem[Williams et~al.(2016)Williams, Drews, Goldfain, Rehg, and
  Theodorou]{mppi_williams2016}
Williams, G., Drews, P., Goldfain, B., Rehg, J.~M., and Theodorou, E.~A.
\newblock Aggressive driving with model predictive path integral control.
\newblock In \emph{IEEE International Conference on Robotics and Automation
  (ICRA)}, 2016.

\bibitem[Wu et~al.(2022)Wu, Escontrela, Hafner, Goldberg, and
  Abbeel]{daydreamer_wu2022}
Wu, P., Escontrela, A., Hafner, D., Goldberg, K., and Abbeel, P.
\newblock Daydreamer: World models for physical robot learning.
\newblock In \emph{Conference on Robot Learning (CoRL)}, 2022.

\bibitem[Wulfmeier et~al.(2021)Wulfmeier, Rao, Hafner, Lampe, Abdolmaleki,
  Hertweck, Neunert, Tirumala, Siegel, Heess, and
  Riedmiller]{ho2_wulfmeier2021}
Wulfmeier, M., Rao, D., Hafner, R., Lampe, T., Abdolmaleki, A., Hertweck, T.,
  Neunert, M., Tirumala, D., Siegel, N., Heess, N. M.~O., and Riedmiller, M.~A.
\newblock Data-efficient hindsight off-policy option learning.
\newblock In \emph{International Conference on Machine Learning (ICML)}, 2021.

\bibitem[Xie et~al.(2020)Xie, Bharadhwaj, Hafner, Garg, and
  Shkurti]{lsp_xie2020}
Xie, K., Bharadhwaj, H., Hafner, D., Garg, A., and Shkurti, F.
\newblock Latent skill planning for exploration and transfer.
\newblock In \emph{International Conference on Learning Representations
  (ICLR)}, 2020.

\bibitem[Yarats et~al.(2021)Yarats, Fergus, Lazaric, and
  Pinto]{protorl_yarats2021}
Yarats, D., Fergus, R., Lazaric, A., and Pinto, L.
\newblock Reinforcement learning with prototypical representations.
\newblock In \emph{International Conference on Machine Learning (ICML)}, 2021.

\bibitem[Yu et~al.(2020)Yu, Thomas, Yu, Ermon, Zou, Levine, Finn, and
  Ma]{mopo_yu2020}
Yu, T., Thomas, G., Yu, L., Ermon, S., Zou, J.~Y., Levine, S., Finn, C., and
  Ma, T.
\newblock Mopo: Model-based offline policy optimization.
\newblock In \emph{Neural Information Processing Systems (NeurIPS)}, 2020.

\end{thebibliography}
\bibliographystyle{icml2023}

\newpage
\appendix
\onecolumn

\begin{figure}[t!]
    \centering
    \begin{subfigure}[ht]{0.32\linewidth}
        \centering
        \includegraphics[width=\linewidth]{figures/qual_ant.pdf}
        \caption{Ant}
    \end{subfigure}
    \begin{subfigure}[ht]{0.32\linewidth}
        \centering
        \includegraphics[width=\linewidth]{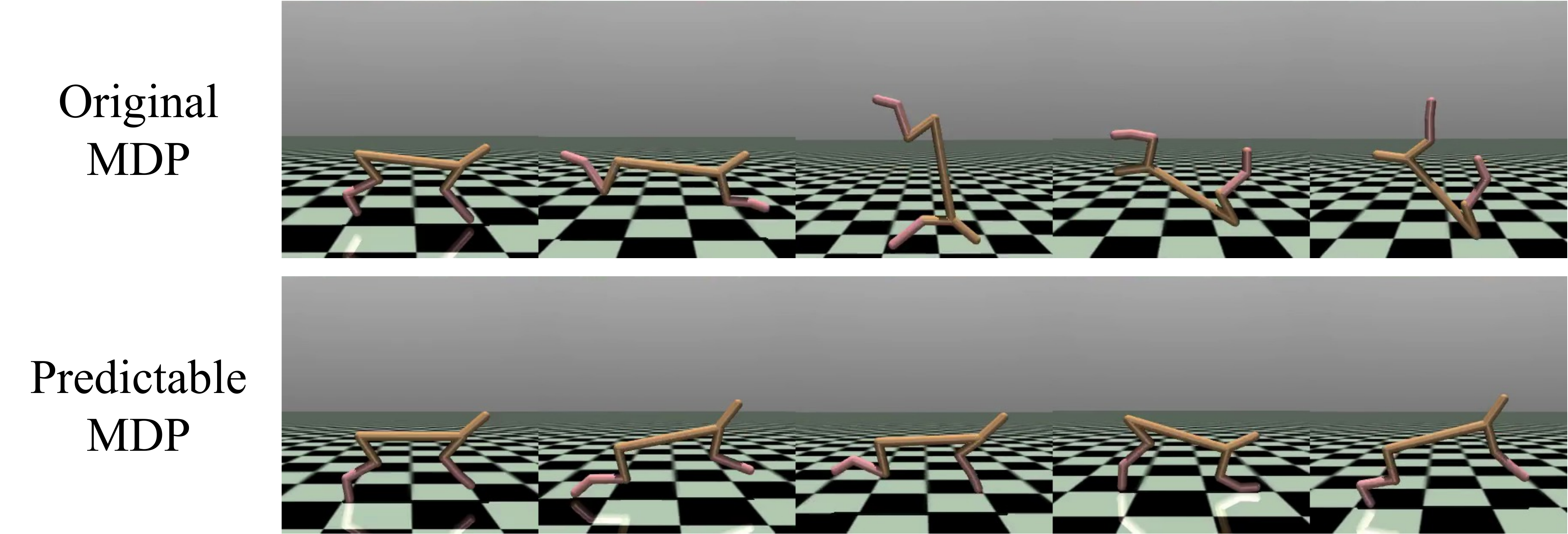}
        \caption{HalfCheetah}
    \end{subfigure}
    \begin{subfigure}[ht]{0.32\linewidth}
        \centering
        \includegraphics[width=\linewidth]{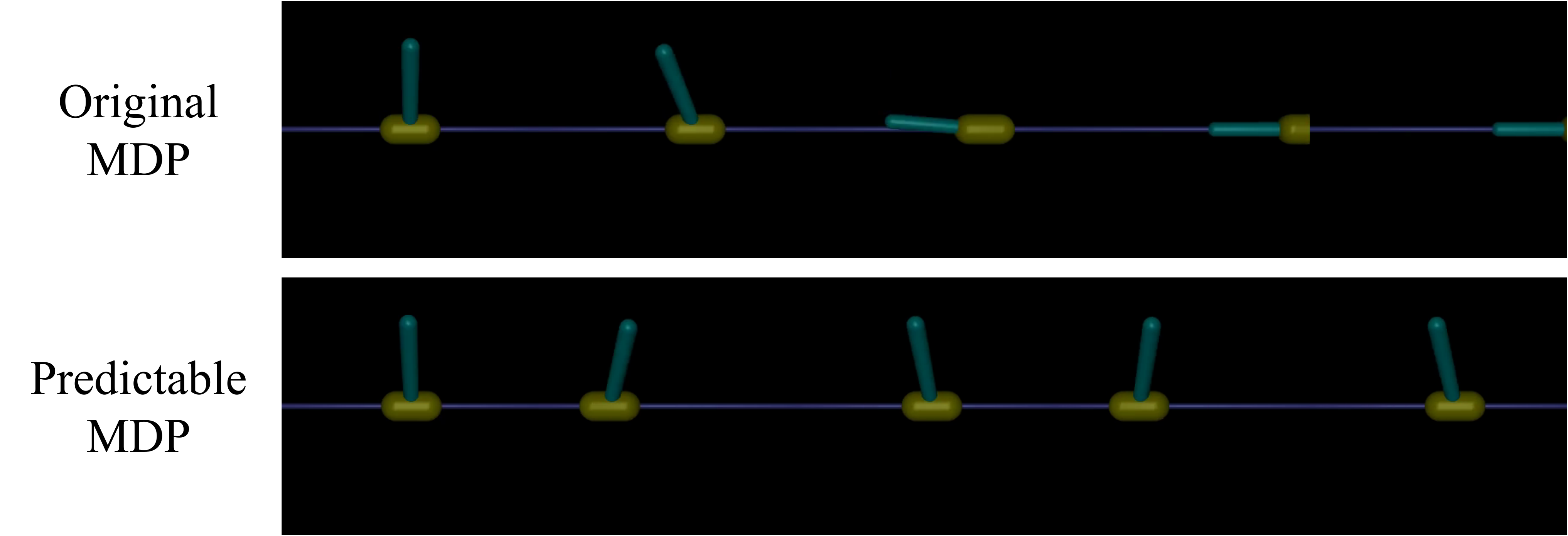}
        \caption{InvertedPendulum}
    \end{subfigure}
    \caption{
    Examples of PMA.
    PMA prevents unpredictable, chaotic actions so that every transition in the latent MDP is maximally predictable.
    Videos are available at \pmaaddress.
    }
    \label{fig:qual}
\end{figure}

\section{Examples of PMA}
\label{sec:appx_qual}
To illustrate the difference between original MDPs and the corresponding predictable MDPs,
we present qualitiative examples of PMA in \Cref{fig:qual}.
In Ant and HalfCheetah,
our predictable MDP only allows actions that are easy to model yet diverse enough to solve downstream tasks,
preventing unpredictable behaviors such as chaotic flipping.
In InvertedPendulum, we find that most of the learned latent actions move the agent in different directions
while maintaining balance,
even without early termination, in order to make the transitions maximally predictable.
Videos are available at \pmaaddress.

\begin{figure}[t!]
    \centering
    \includegraphics[width=0.6\linewidth]{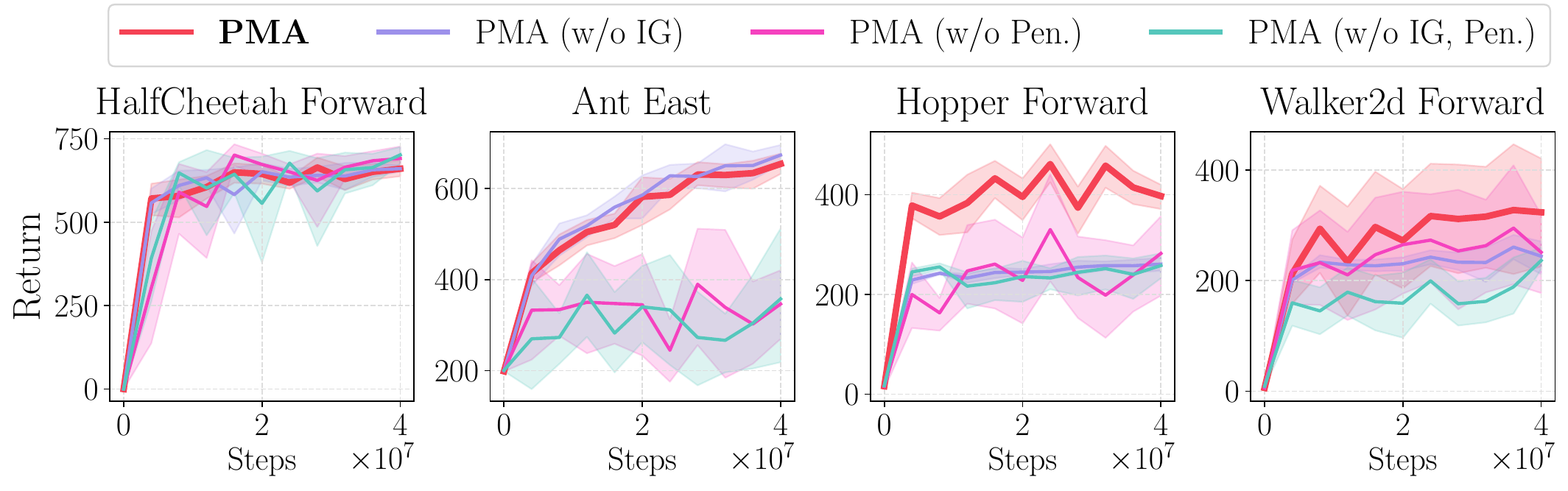}
    \caption{
    Ablation study of the disagreement bonus (``IG'') during unsupervised training
    and disagreement penalty (``Pen.'') during planning.
    The disagreement penalty generally stabilizes training, and the disagreement bonus improves performance.
    }
    \label{fig:abl}
\end{figure}

\section{Ablation study}
\label{sec:appx_abl}
To evaluate the relative importance of each component of PMA,
we ablate the information gain (disagreement bonus) term during unsupervised training
and the disagreement penalty during periodic MPPI planning,
and report performances in \Cref{fig:abl}.
While there are small performance differences between the settings in HalfCheetah,
these components improve and stabilize the performances in the other more complex environments
by encouraging exploration and preventing distributional shifts.

\section{Approximating Information Gain with Ensemble Disagreement}
\label{sec:appx_approx}

In this section,
we provide a justification for our use of ensemble disagreement as a way to approximate
the information gain term $I(\bm{S}';\bm{\Theta}|\bm{\gD}, \bm{Z})$ in \Cref{eq:pma_decomp}.
First, let the random variable $\hat{\bm{S}}' \sim \hat{p}_z(\cdot|\bm{S}, \bm{Z}; \bm{\Theta})$ denote
the \emph{predicted} state under a model with parameters $\bm{\Theta}$.
Since $\bm{S}' \to \bm{\Theta} \to \hat{\bm{S}}'$ forms a Markov chain conditioned on $\bm{\gD}$ and $\bm{Z}$ in our Bayesian setting,
we get the following lower bound by the data processing inequality:
\begin{align}
    &I(\bm{S}';\bm{\Theta}|\bm{\gD}, \bm{Z}) \geq I(\bm{S}';\hat{\bm{S}}'|\bm{\gD}, \bm{Z}) \\
    &= H(\hat{\bm{S}}'|\bm{\gD}, \bm{Z}) - H(\hat{\bm{S}}'|\bm{\gD}, \bm{Z}, \bm{S}'). \label{eq:infogain}
\end{align}
Now, we approximate the model posterior with 
an ensemble of $E$ predictive models, $\{\hat{p}_z(\vs'|\vs, \vz; \bm{\theta}_i)\}_{i \in [E]}$
with $p(\bm{\theta}|\gD) = \frac{1}{E} \sum_i \delta (\bm{\theta} - \bm{\theta}_i)$,
where each model is represented as a conditional Gaussian
with the mean given by a neural network and a unit diagonal covariance,
$\vs' \sim \gN (\bm{\mu}(\vs, \vz; \bm{\theta}_i), I)$.
The terms in \Cref{eq:infogain} measure the uncertainty in the predicted next state
before and after observing the outcome $\bm{S}'$, respectively.
Yet, it is still intractable because there is no closed-form formulation
for the differential entropy of a mixture of Gaussian distributions.
Hence, we further simplify these terms as follows.
First, we assume that the second term in \Cref{eq:infogain} has a negligible effect on the objective,
which roughly corresponds to assuming a low training error
(\ie, if we know the value of $\bm{S}'$ and we update the models, they should agree on $\hat{\bm{S}}'$, or at least have similar error).
This assumption might not hold in heteroskedastic environments but is otherwise very convenient.
Next, we empirically substitute the first term in \Cref{eq:infogain}
with the variance of the ensemble means with a coefficient $\beta$,
$\E [\beta \cdot \mathrm{Tr}[\sV_i[\bm{\mu}(\vs, \vz; \bm{\theta}_i)]]]$,
based on the fact that they both are correlated to the uncertainty in $\hat{\bm{S}}'$ \citep{p2e_sekar2020}:
if the predictions of the ensemble models are very different from one another (\ie, the variance is large),
the marginal entropy of $\hat{\bm{S}}'$ will also be large, and vice versa.
As a result, we get our approximation in \Cref{eq:disag}.
We also refer to prior works \citep{max_shyam2019,p2e_sekar2020,rp1_ball2020,disdain_strouse2022}
for similar connections between ensemble disagreement and information gain.

\section{Theoretical Results}
\label{sec:appx_theory}

\subsection{Technical Lemmas}
For an MDP $\gM := (\gS, \gA, \mu, p, r)$\footnote{In our theoretical analyses, we assume that the state and action spaces are finite for simplicity.}
and a policy $\pi$,
we first define the discounted state and state-action distributions, and state the bellman flow constraint lemma.
\begin{definition}
(Discounted state distribution) $d^\pi(\vs) := (1-\gamma) \sum_{t=0}^{\infty} \gamma^t P(\vs_t = \vs|\mu, p, \pi)$.
\end{definition}
\begin{definition}
(Discounted state-action distribution) $d^\pi(\vs, \va) := \pi(\va|\vs) d^\pi(\vs)$.
\end{definition}
\begin{lemma} \label{thm:flow} (Bellman flow constraint)
\begin{align}
    d^\pi(\vs) = (1-\gamma)\mu(\vs) + \gamma \sum_{\vs^- \in \gS, \va^- \in \gA} p(\vs|\vs^-, \va^-) d^\pi(\vs^-, \va^-).
\end{align}
\end{lemma}
\begin{proof}
\begin{align}
    d^\pi(\vs)
    &= (1-\gamma) (P(\vs_0 = \vs) + \gamma P(\vs_1 = \vs) + \gamma^2 P(\vs_2 = \vs) + \cdots) \\
    &= (1-\gamma) \mu(\vs) + \gamma (1-\gamma) (P(\vs_1 = \vs) + \gamma P(\vs_2 = \vs) + \cdots) \\
    &= (1-\gamma) \mu(\vs) + \gamma (1-\gamma) \sum_{\vs^-, \va^-} p(\vs|\vs^-, \va^-) (P(\vs_0 = \vs^-, \va_0 = \va^-)
        + \gamma P(\vs_1 = \vs^-, \va_1 = \va^-) + \cdots) \\
    &= (1-\gamma) \mu(\vs) + \gamma \sum_{\vs^-, \va^-} p(\vs|\vs^-, \va^-) d^\pi(\vs^-, \va^-).
\end{align}
\end{proof}

Now, we consider two MDPs with different transition dynamics, $\gM_1 := (\gS, \gA, r, \mu, p_1)$ and $\gM_2 := (\gS, \gA, r, \mu, p_2)$,
and two policies, $\pi_1$ and $\pi_2$.
We denote the expected return of $\pi$ as $J_\gM(\pi) := \frac{1}{1-\gamma}\E_{(\vs, \va) \sim d^\pi(\vs, \va)}[r(\vs, \va)]$
and the maximum reward as $R := \max_{\vs \in \gS, \va \in \gA} r(\vs, \va)$.
We can bound their performance difference as follows.
\begin{lemma}
\label{thm:perfdiff}
If the total variation distances of the dynamics and the policies are bounded as
\begin{align}
    \E_{(\vs, \va) \sim d_1^{\pi_1}(\vs, \va)} [\TV (p_1(\cdot|\vs, \va) \| p_2(\cdot|\vs, \va))] &\leq \epsilon_m, \label{eq:epsilon_m} \\
    \E_{\vs \sim d_1^{\pi_1}(\vs)} [\TV (\pi_1(\cdot|\vs) \| \pi_2(\cdot|\vs))] &\leq \epsilon_\pi, \label{eq:epsilon_pi}
\end{align}
their performance difference satisfies the following inequality:
\begin{align}
    \left|J_{\gM_1}(\pi_1) - J_{\gM_2}(\pi_2)\right| \leq \frac{R}{(1-\gamma)^2} (2 \gamma \epsilon_m + 2 \epsilon_\pi).
    \label{eq:perfdiff}
\end{align}
\end{lemma}
\begin{proof}
We first bound the difference in their discounted state-action distributions.
\begin{align}
    &\sum_{\vs \in \gS, \va \in \gA} \left|d_1^{\pi_1}(\vs, \va) - d_2^{\pi_2}(\vs, \va)\right| \\
    &= \sum_{\vs \in \gS, \va \in \gA} \left|\pi_1(\va|\vs) d_1^{\pi_1}(\vs) - \pi_2(\va|\vs) d_2^{\pi_2}(\vs)\right| \\
    &\leq \sum_{\vs \in \gS, \va \in \gA} \left|\pi_1(\va|\vs) d_1^{\pi_1}(\vs) - \pi_2(\va|\vs) d_1^{\pi_1}(\vs)\right|
        + \sum_{\vs \in \gS, \va \in \gA} \left|\pi_2(\va|\vs) d_1^{\pi_1}(\vs) - \pi_2(\va|\vs) d_2^{\pi_2}(\vs)\right| \\
    &= \sum_{\vs \in \gS} d_1^{\pi_1}(\vs) \sum_{\va \in \gA} \left|\pi_1(\va|\vs) - \pi_2(\va|\vs)\right|
        + \sum_{\vs \in \gS} \left|d_1^{\pi_1}(\vs) - d_2^{\pi_2}(\vs)\right| \\
    &\leq \sum_{\vs \in \gS} \left|d_1^{\pi_1}(\vs) - d_2^{\pi_2}(\vs)\right| + 2 \epsilon_\pi \\
    &= \gamma \sum_{\vs \in \gS} \left|\sum_{\vs^- \in \gS, \va^- \in \gA} \left(
        p_1(\vs|\vs^-, \va^-)d_1^{\pi_1}(\vs^-, \va^-) - p_2(\vs|\vs^-, \va^-)d_2^{\pi_2}(\vs^-, \va^-)
    \right)\right| + 2 \epsilon_\pi \label{eq:perfdiff_eq1} \\
    &\leq \gamma \sum_{\vs^- \in \gS, \va^- \in \gA, \vs \in \gS} \left|
        p_1(\vs|\vs^-, \va^-)d_1^{\pi_1}(\vs^-, \va^-) - p_2(\vs|\vs^-, \va^-)d_2^{\pi_2}(\vs^-, \va^-)
    \right| + 2 \epsilon_\pi \\
    &\leq \gamma \sum_{\vs^- \in \gS, \va^- \in \gA, \vs \in \gS} \left|
        p_1(\vs|\vs^-, \va^-)d_1^{\pi_1}(\vs^-, \va^-) - p_2(\vs|\vs^-, \va^-)d_1^{\pi_1}(\vs^-, \va^-)
    \right| \nonumber \\
    &\quad + \gamma \sum_{\vs^- \in \gS, \va^- \in \gA, \vs \in \gS} \left|
        p_2(\vs|\vs^-, \va^-)d_1^{\pi_1}(\vs^-, \va^-) - p_2(\vs|\vs^-, \va^-)d_2^{\pi_2}(\vs^-, \va^-)
    \right| + 2 \epsilon_\pi \\
    &= \gamma \sum_{\vs^- \in \gS, \va^- \in \gA} d_1^{\pi_1}(\vs^-, \va^-) \sum_{\vs \in \gS}
        \left| p_1(\vs|\vs^-, \va^-) - p_2(\vs|\vs^-, \va^-) \right| \nonumber \\
    &\quad+ \gamma \sum_{\vs^- \in \gS, \va^- \in \gA} \left| d_1^{\pi_1}(\vs^-, \va^-) - d_2^{\pi_2}(\vs^-, \va^-) \right|
        + 2 \epsilon_\pi \\
    &\leq 2 \gamma \epsilon_m + 2 \epsilon_\pi + \gamma \sum_{\vs^- \in \gS, \va^- \in \gA}
        \left| d_1^{\pi_1}(\vs^-, \va^-) - d_2^{\pi_2}(\vs^-, \va^-) \right| \\
    &= 2 \gamma \epsilon_m + 2 \epsilon_\pi + \gamma \sum_{\vs \in \gS, \va \in \gA}
        \left| d_1^{\pi_1}(\vs, \va) - d_2^{\pi_2}(\vs, \va) \right|,
\end{align}
which implies
\begin{align}
    \sum_{\vs \in \gS, \va \in \gA} \left|d_1^{\pi_1}(\vs, \va) - d_2^{\pi_2}(\vs, \va)\right|
    \leq \frac{1}{1-\gamma}(2 \gamma \epsilon_m + 2 \epsilon_\pi),
\end{align}
where we use \Cref{thm:flow} in \Cref{eq:perfdiff_eq1}.
Hence, we obtain
\begin{align}
    \left|J_{\gM_1}(\pi_1) - J_{\gM_2}(\pi_2)\right|
    &= \frac{1}{1-\gamma} \left| \sum_{\vs \in \gS, \va \in \gA} (d_1^{\pi_1}(\vs, \va) - d_2^{\pi_2}(\vs, \va)) r(\vs, \va) \right| \\
    &\leq \frac{R}{1-\gamma} \sum_{\vs \in \gS, \va \in \gA} |d_1^{\pi_1}(\vs, \va) - d_2^{\pi_2}(\vs, \va)| \\
    &\leq \frac{R}{(1-\gamma)^2} (2 \gamma \epsilon_m + 2 \epsilon_\pi).
\end{align}
\end{proof}
\Cref{eq:perfdiff} is the same bound as Lemma B.3 in \citet{mbpo_janner2019}, but we use a milder assumption in \Cref{eq:epsilon_pi},
which only assumes that the expectation (not the maximum) of the total variation distance between the policies is bounded.

\subsection{MBRL Performance Bound}
\label{sec:appx_mbrl_bound}

We first present the performance bound of a policy $\pi$ in the original MDP $\gM = (\gS, \gA, \mu, p, r)$ and its model-based MDP \citep{mbpo_janner2019}.
We denote the model-based MDP with a learned predictive model $\hat{p}$ as $\hat{\gM} = (\gS, \gA, \mu, \hat{p}, r)$.
We assume that the model $\hat{p}$ is trained on a dataset $\gD$, which is collected by a data-collecting policy $\pi_\gD$.
\begin{theorem}
\label{thm:model_bound}
For any policy $\pi$,
if the total variation distances between (\emph{i}) the true dynamics $p$ and the learned model $\hat{p}$
and (\emph{ii}) the policy $\pi$ and the data-collection policy $\pi_\gD$ are bounded as
\begin{align}
    \E_{(\vs, \va) \sim d^{\pi_\gD}(\vs, \va)} [\TV (p(\cdot|\vs, \va) \| \hat{p}(\cdot|\vs, \va))] &\leq \epsilon_m, \\
    \E_{\vs \sim d^{\pi_\gD}(\vs)} [\TV (\pi(\cdot|\vs) \| \pi_\gD(\cdot|\vs))] &\leq \epsilon_\pi,
\end{align}
the performance difference of $\pi$ between $\gM$ and $\hat{\gM}$ satisfies the following inequality:
\begin{align}
    |J_{\gM}(\pi) - J_{\hat{\gM}}(\pi)| \leq \frac{R}{(1-\gamma)^2}(4 \epsilon_\pi + 2\gamma \epsilon_m).
\end{align}
\end{theorem}
\begin{proof}
From \Cref{thm:perfdiff}, we get
\begin{align}
    |J_{\gM}(\pi) - J_{\hat{\gM}}(\pi)|
    &\leq |J_{\gM}(\pi) - J_{\gM}(\pi_\gD)| + |J_{\gM}(\pi_\gD) + J_{\hat{\gM}}(\pi)| \\
    &\leq \frac{R}{(1-\gamma)^2}(2 \epsilon_\pi) + \frac{R}{(1-\gamma)^2}(2 \epsilon_\pi + 2\gamma \epsilon_m) \\
    &= \frac{R}{(1-\gamma)^2}(4 \epsilon_\pi + 2\gamma \epsilon_m).
\end{align}
\end{proof}

\subsection{PMA Performance Bound}
\label{sec:appx_pma_bound}

We provide the performance bound of our predictable MDP abstraction.
We denote our predictable latent MDP as $\gM_P := (\gS, \gZ, \mu, p_z, r)$ with the latent action space $\gZ$
and the reward function $r(\vs, \vz) = \sum_a \pi_z(\va|\vs, \vz) r(\vs, \va)$,
and its model-based MDP as $\hat{\gM}_P := (\gS, \gZ, \mu, \hat{p}_z, r)$.
We assume that the model $\hat{p}_z$ is trained on a dataset $\gD_P$ collected by a data-collecting policy $\pi_{\gD_P}$ in the latent MDP.

For the performance bound of a policy $\pi(\va|\vs)$ between the original MDP and the predictable model-based MDP,
we independently tackle the performance losses caused by (\emph{i}) MDP abstraction and (\emph{ii}) model learning.
For the first part, we take a similar approach to \citet{nearoptimal_nachum2019,opal_ajay2021}.
For any $\vs \in \gS$, $\va \in \gA$, and a probability distribution $\phi: \gS \times \gA \to \gP(\gZ)$,
we define the following state distribution in $\gS \times \gA \to \gP(\gS)$:
\begin{align}
    p_z^\phi(\cdot|\vs, \va) := \sum_\vz \phi(\vz|\vs, \va) p_z(\cdot|\vs, \vz).
\end{align}
Also, we define the optimal $\vz$ distribution that best mimics the original transition distribution:
\begin{align}
    \phi^*(\cdot|\vs, \va) := \argmin_{\phi \in \gS \times \gA \to \gP(\gZ)} \TV(p(\cdot|\vs, \va) \| p_z^\phi(\cdot|\vs, \va)).
\end{align}
For a policy $\pi$ in the original MDP, we define its corresponding optimal latent policy as follows:
\begin{align}
    \pi_z^{\phi^*}(\cdot|\vs) := \sum_{\va \in \gA} \pi(\va|\vs) \phi^*(\cdot|\vs, \va).
\end{align}
Intuitively, this policy produces latent action distributions that mimic the next state distributions of $\pi$ as closely as possible.
Now, we state the performance bound of $\pi$ between the original MDP and the predictable latent MDP.
\begin{lemma}
\label{thm:abs_bound}
For any policy $\pi$,
if the total variation distance between the original transition dynamics and the optimal latent dynamics is bounded as
\begin{align}
    \E_{(\vs, \va) \sim d^{\pi}(\vs, \va)} [\TV (p(\cdot|\vs, \va) \| p_z^{\phi^*}(\cdot|\vs, \va))] &\leq \epsilon_a,
\end{align}
the performance difference between the original MDP and the predictable latent MDP satisfies the following inequality:
\begin{align}
    |J_\gM(\pi) - J_{\gM_P}(\pi_z^{\phi^*})| \leq \frac{2R \gamma \epsilon_a}{(1 - \gamma)^2}. \label{eq:abs_bound}
\end{align}
\end{lemma}
\begin{proof}
The next state distribution of $\pi_z^{\phi^*}$ at state $\vs \in \gS$ in the latent MDP can be written as follows:
\begin{align}
    p_z(\cdot|\vs)
    &= \sum_{\vz \in \gZ} \pi_z^{\phi^*}(\vz|\vs) p_z(\cdot|\vs, \vz) \\
    &= \sum_{\vz \in \gZ} \sum_{\va \in \gA} \pi(\va|\vs) \phi^*(\vz|\vs, \va) p_z(\cdot|\vs, \vz) \\
    &= \sum_{\va \in \gA} \pi(\va|\vs) \sum_{\vz \in \gZ} \phi^*(\vz|\vs, \va) p_z(\cdot|\vs, \vz) \\
    &= \sum_{\va \in \gA} \pi(\va|\vs) p_z^{\phi^*}(\cdot|\vs, \va).
\end{align}
Hence, $J_{\gM_P}(\pi_z^{\phi^*})$ is equal to $J_{\gM^{\phi^*}}(\pi)$,
where $\gM^{\phi^*}$ is defined as $(\gS, \gA, \mu, p_z^{\phi^*}, r)$.
Then, \Cref{eq:abs_bound} follows from \Cref{thm:perfdiff}.
\end{proof}

Next, we bound the performance difference of $\pi_z^{\phi^*}$ between the latent MDP and latent model-based MDP.
\begin{lemma}
\label{thm:abs_model_bound}
If the total variation distances between (\emph{i}) the true dynamics $p_z$ and the learned model $\hat{p}_z$
and (\emph{ii}) the policy $\pi_z^{\phi^*}$ and the data-collection policy $\pi_{\gD_P}$ are bounded as
\begin{align}
    \E_{(\vs, \vz) \sim d^{\pi_{\gD_P}}(\vs, \vz)} [\TV (p_z(\cdot|\vs, \vz) \| \hat{p}_z(\cdot|\vs, \vz))] &\leq \epsilon_m', \\
    \E_{\vs \sim d^{\pi_{\gD_P}}(\vs)} [\TV (\pi_z^{\phi^*}(\cdot|\vs) \| \pi_{\gD_P}(\cdot|\vs))] &\leq \epsilon_\pi',
\end{align}
the performance difference of $\pi_z^{\phi^*}$ between $\gM_P$ and $\hat{\gM}_P$ satisfies the following inequality:
\begin{align}
    |J_{\gM_P}(\pi_z^{\phi^*}) - J_{\hat{\gM}_P}(\pi_z^{\phi^*})| \leq \frac{R}{(1-\gamma)^2}(4 \epsilon_\pi' + 2\gamma \epsilon_m').
\end{align}
\end{lemma}
\begin{proof}
The proof is the same as \Cref{thm:model_bound}.
\end{proof}

Now, we get the following performance bound of PMA:
\begin{theorem} (PMA performance bound)
If the abstraction loss, the model error, and the policy difference are bounded as follows:
\begin{align}
    \E_{(\vs, \va) \sim d^{\pi}(\vs, \va)} [\TV (p(\cdot|\vs, \va) \| p_z^{\phi^*}(\cdot|\vs, \va))] &\leq \epsilon_a, \\
    \E_{(\vs, \vz) \sim d^{\pi_{\gD_P}}(\vs, \vz)} [\TV (p_z(\cdot|\vs, \vz) \| \hat{p}_z(\cdot|\vs, \vz))] &\leq \epsilon_m', \\
    \E_{\vs \sim d^{\pi_{\gD_P}}(\vs)} [\TV (\pi_z^{\phi^*}(\cdot|\vs) \| \pi_{\gD_P}(\cdot|\vs))] &\leq \epsilon_\pi',
\end{align}
the performance difference of $\pi$ between the original MDP and the predictable latent model-based MDP is bounded as:
\begin{align}
    |J_{\gM}(\pi) - J_{\hat{\gM}_P}(\pi_z^{\phi^*})| \leq \frac{R}{(1-\gamma)^2}(2 \gamma \epsilon_a + 4 \epsilon_\pi' + 2\gamma \epsilon_m').
\end{align}
\end{theorem}
\begin{proof}
From \Cref{thm:abs_bound} and \Cref{thm:abs_model_bound}, we obtain
\begin{align}
    |J_{\gM}(\pi) - J_{\hat{\gM}_P}(\pi_z^{\phi^*})|
    &\leq |J_{\gM}(\pi) - J_{\gM_P}(\pi_z^{\phi^*})| + |J_{\gM_P}(\pi_z^{\phi^*}) - J_{\hat{\gM}_P}(\pi_z^{\phi^*})| \\
    &\leq \frac{R}{(1-\gamma)^2}(2 \gamma \epsilon_a) + \frac{R}{(1-\gamma)^2}(4 \epsilon_\pi' + 2\gamma \epsilon_m') \\
    &= \frac{R}{(1-\gamma)^2}(2 \gamma \epsilon_a + 4 \epsilon_\pi' + 2\gamma \epsilon_m').
\end{align}
\end{proof}

\section{Theoretical Comparison between PMA and DADS}
\label{sec:appx_pma_dads}

In this section, we theoretically compare the objectives of PMA and DADS in terms of
mutual information maximization, entropy approximation, and state coverage.

\paragraph{Mutual information maximization.}
\begin{wrapfigure}{r}{0.3\textwidth}
  \centering
  \raisebox{0pt}[\dimexpr\height-1.0\baselineskip\relax]{
    \begin{subfigure}[t]{1.0\linewidth}
      \includegraphics[width=\linewidth]{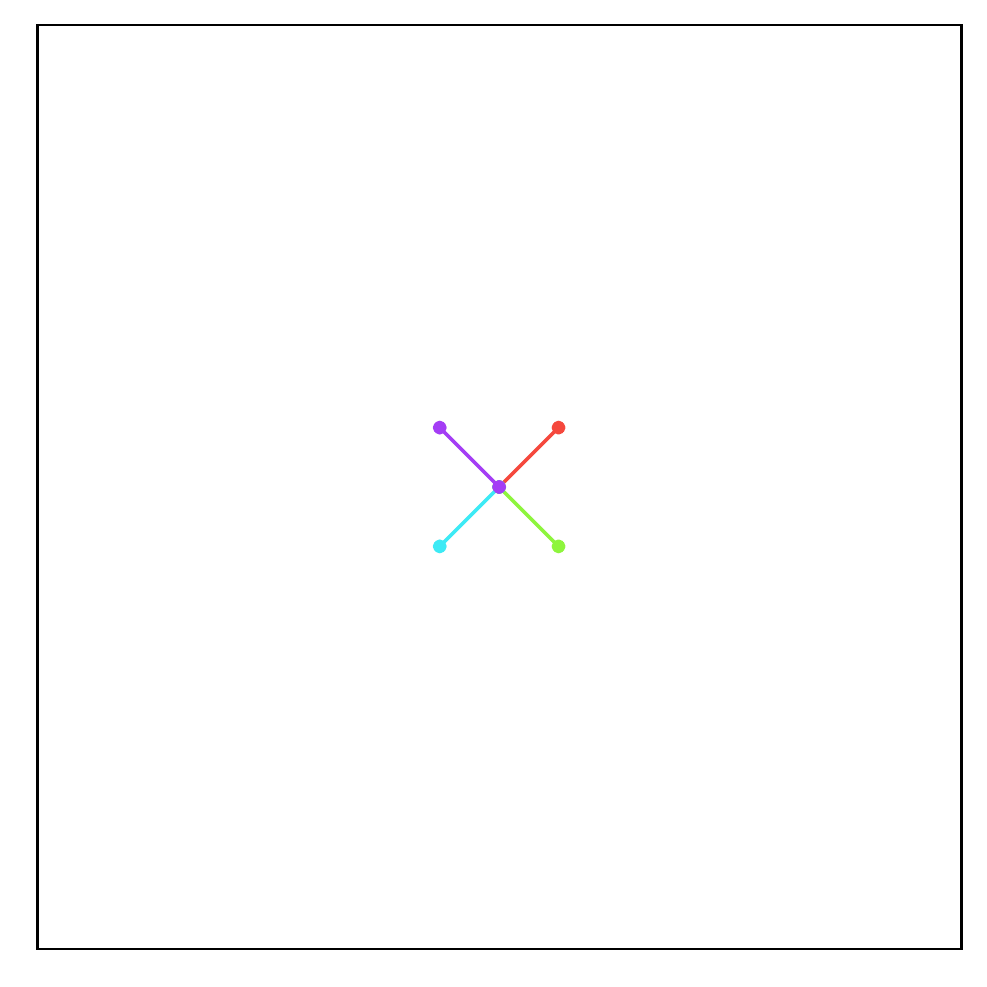}
    \end{subfigure}
  }
  \caption{
  An example of the trajectories of a DADS policy ($Z = 4$).
  }
  \label{fig:dads_skills}
\end{wrapfigure}
We first compare the mutual information (MI) term $I(\bm{S}';\bm{Z}|\bm{S})$ in the objectives of PMA and DADS.
Here, we ignore the information gain term of PMA, \ie, $\beta = 0$, which we will discuss later.
Also, in order to simplify the analysis of information-theoretic objectives, we assume that the latent action space $\gZ$ is discrete,
\ie $\gZ = [Z]$.
With these assumptions, the difference between PMA and DADS mainly lies in the exploration policy, $\pi_e(\vz|\vs)$.
While DADS first samples a latent action $\vz$ at the beginning of each rollout ($\pi_e(\cdot|\vs_0) = \text{Unif}(\gZ)$)
and persists it throughout the entire episode ($\pi_e(\vz_t = \vz_{t-1}|\vs_t) = 1$),
PMA always uses a uniform random policy ($\pi_e(\cdot|\vs) = \text{Unif}(\gZ)$) because we have assumed $\beta = 0$.

We first consider the following upper bound of the MI objective, $I(\bm{S}';\bm{Z}|\bm{S})$,
\begin{align}
    \max_{\pi_e,\pi_z} I(\bm{S}';\bm{Z}|\bm{S})
    &= \max_{\pi_e,\pi_z} H(\bm{Z}|\bm{S}) - H(\bm{Z}|\bm{S}, \bm{S}') \\
    &\leq \max_{\pi_e,\pi_z} H(\bm{Z}|\bm{S}). \label{eq:mi_ub}
\end{align}
\Cref{eq:mi_ub} achieves its maximum of $\log Z$ when $p^{\pi_e,\pi_z}(\vz|\vs)$ is a uniform random distribution,
where $p^{\pi_e,\pi_z}(\vs, \vz)$ is the state-latent action distribution from the policies.
PMA uses $\pi_e(\cdot|\vs) = \text{Unif}(\gZ)$, which precisely corresponds to this optimal condition.
On the other hand, DADS's distribution can be rewritten as
\begin{align}
    p^{\pi_e,\pi_z}(\vz|\vs) = \frac{p^{\pi_e,\pi_z}(\vs|\vz) p^{\pi_e,\pi_z}(\vz)}{p^{\pi_e,\pi_z}(\vs)}
    \propto \frac{p^{\pi_e,\pi_z}(\vs|\vz)}{p^{\pi_e,\pi_z}(\vs)}.
\end{align}
Unlike PMA, this does not necessarily correspond to a uniform distribution.
For example,
at the red dot state in \Cref{fig:dads_skills},
we can see that $p^{\pi_e,\pi_z}(\vs|\vz)$ is nonzero for the corresponding latent action
but is zero for the other latent actions.
This could make DADS suboptimal in terms of MI maximization.

\paragraph{Entropy approximation.}
A similar difference can be found in the approximation of the marginal entropy term $H(\bm{S}'|\bm{S})$.
Both PMA and DADS use the following approximation:
\begin{align}
    H(\bm{S}'|\bm{S}) &= \E \left[-\log p^{\pi_e,\pi_z}(\vs'|\vs) \right] \\
    &= \E \left[-\log \int p^{\pi_e,\pi_z}(\vz|\vs) p^{\pi_e,\pi_z}(\vs'|\vs, \vz) d\vz \right] \\
    &\approx \E \left[-\log \int u(\vz) p^{\pi_e,\pi_z}(\vs'|\vs, \vz) d\vz \right] \\
    &\approx \E\left[-\log \left( \frac{1}{L} \sum_{i=1}^L p^{\pi_e,\pi_z}(\vs'|\vs, \vz_i) \right) \right],
\end{align}
where it approximates $p^{\pi_e,\pi_z}(\vz|\vs)$ to a uniform distribution $u(\vz)$, and $\vz_i$'s are sampled from $u(\cdot)$.
While this approximation may not be accurate in DADS due to the same reason above,
PMA satisfies $p^{\pi_e,\pi_z}(\vz|\vs) = u(\vz)$, making the approximation in the $\log$ exact.

\paragraph{State coverage.}
Another difference between PMA and DADS is the presence of the information gain term in the PMA objective.
This is because the MI term alone does not necessarily cover the state space
since MI is invariant to any invertible transformations of the input random variables,
\ie, $I(\bm{X}; \bm{Y}) = I(f(\bm{X}); g(\bm{Y}))$ for any random variables $\bm{X}$ and $\bm{Y}$,
and invertible functions $f$ and $g$.
As a result, MI can be fully maximized with limited state coverage and does not necessarily encourage exploration
\citep{edl_campos2020,disdain_strouse2022,lsd_park2022},
which necessitates another term for maximizing the state-action coverage in PMA.

\section{Training Procedures}
\label{sec:appx_training}

\begin{algorithm}[t!]
    \caption{Predictable MDP Abstraction (PMA)}
    \begin{algorithmic}[1]
        \STATE Initialize action decoder $\pi_z(\va|\vs, \vz)$, VLB predictive model $\hat{p}_z(\vs'|\vs, \vz; \bm{\phi})$,
        ensemble predictive models $\{\hat{p}_z(\vs'|\vs, \vz; \bm{\theta_i})\}$, (optional) exploration policy $\pi_e(\vz|\vs)$,
        on-policy stochastic buffer $\gD_S$, on-policy deterministic buffer $\gD_D$, replay buffer $\gD$
        \FOR{$i \gets 1$ to (\# epochs)}
            \FOR{$j \gets 1$ to $\text{(\# steps per epoch)} / 2$}
                \STATE Sample latent action $\vz \sim \pi_e(\vz|\vs)$
                \STATE Sample action $\va \sim \pi_z(\va|\vs, \vz)$
                \STATE Add transition $(\vs, \vz, \va, \vs')$ to $\gD_S$, $\gD$
            \ENDFOR
            \FOR{$j \gets 1$ to $\text{(\# steps per epoch)} / 2$}
                \STATE Sample latent action $\vz \sim \pi_e(\vz|\vs)$
                \STATE Compute deterministic action $\va = \E[\pi_z(\cdot|\vs, \vz)]$
                \STATE Add transition $(\vs, \vz, \va, \vs')$ to $\gD_D$
            \ENDFOR
            \STATE Fit VLB predictive model using mini-batches from $\gD_S$
            \STATE Fit ensemble predictive models using mini-batches from $\gD_D$
            \STATE Train action decoder with $r(\vs, \vz, \va, \vs') = \log \hat{p}_z(\vs'|\vs, \vz; \bm{\phi})
            - \log \frac{1}{L} \sum_{i=1}^{L} \hat{p}_z(\vs'|\vs, \vz_i; \bm{\phi})
            + \beta \cdot \mathrm{Tr}[\sV_i[\bm{\mu}(\vs, \vz; \bm{\theta}_i)]]$ with SAC using mini-batches from $\gD$
            \STATE (Optional) Train exploration policy $\pi_e$ with SAC using mini-batches from $\gD_S$
            \STATE Clear on-policy buffers $\gD_S$, $\gD_D$
        \ENDFOR
    \end{algorithmic}
    \label{alg:pma}
\end{algorithm}

\subsection{PMA Training Procedure}
\label{sec:appx_training_pma}

PMA is trained with SAC \citep{saces_haarnoja2018}. We describe several additional training details of PMA.

\paragraph{Replay buffer.}
Since we jointly train both the action decoder and the model,
we need to be careful about using old, off-policy samples to train the components of PMA.
While we can use old samples to train the action decoder $\pi_z(\va|\vs, \vz)$
as long as we recompute the intrinsic reward (because SAC is an off-policy algorithm),
we cannot use old samples to train the predictive models or the exploration policy $\pi_e(\vz|\vs)$.
Hence, we use a replay buffer only for the action decoder, and train the other components with on-policy data.

\paragraph{Sampling strategy.}
PMA has two different kinds of predictive models: the VLB predictive model $\hat{p}_z(\vs'|\vs, \vz; \bm{\phi})$
to approximate $I(\bm{S'};\bm{Z}|\bm{S})$,
and an ensemble of $E$ models $\{\hat{p}_z(\vs'|\vs, \vz; \bm{\theta_i})\}$
to approximate $I(\bm{S}';\bm{\Theta}|\bm{\gD}, \bm{Z})$.
While one may just simply use the mean of the ensemble outputs for the VLB predictive model,
we find that it is helpful to train them separately with different sampling strategies.
Specifically, we train the VLB predictive model with stochastic trajectories
and the ensemble models with deterministic trajectories.
Since we always use deterministic actions from the pre-trained action decoder during the test time,
it is beneficial to have a latent predictive model trained from such deterministic actions.
As such, at each epoch, we sample a half of the epoch transitions using deterministic actions to train the ensemble models
and the other half using stochastic actions to train the other components.
At test time, we use the mean of the ensemble model's outputs for model-based planning or model-based RL.

We summarize the full training procedure of PMA in \Cref{alg:pma}.

\subsection{MPPI Training Procedure}
\label{sec:appx_training_mppi}

\begin{algorithm}[t!]
    \caption{MPPI with PMA}
    \begin{algorithmic}[1]
        \STATE Initialize mean parameter $\bm{\mu}_{0:T-1}$
        \FOR{$t \gets 0$ to $T - 1$}
            \FOR{$m \gets 0$ to $M - 1$}
                \STATE Sample $N$ latent action sequences $\vz_{t:t+H-1}^{(i)} \sim \gN(\bm{\mu}_{t:t+H-1}, \bm{\Sigma})$ for $i \in [N]$
                \STATE Compute sum of predicted rewards $\{R^{(i)} := \sum_{j=t}^{t+H-1} \hat{r}_j^{(i)} \}$
                using predicted states from latent predictive model $\hat{p}_z(\vs'|\vs, \vz)$
                \STATE Update $\bm{\mu}_{t:t+H-1}$ using \Cref{eq:mppi_update}
            \ENDFOR
            \STATE Perform single action $\va_t = \E[\pi_z(\cdot|\vs_t, \vz_t)]$ with $\vz_t = \bm{\mu}_t$ and get $\vs_{t+1}$ from environment
        \ENDFOR
    \end{algorithmic}
    \label{alg:mppi}
\end{algorithm}

MPPI \citep{mppi_williams2016} is a zeroth-order planning algorithm based on model predictive control \citep{mpc_richalet1978},
which finds an optimal action sequence via iterative refinement of randomly sampled actions.
Specifically,
at step $t$,
MPPI aims to find the optimal (latent) action sequence
$(\vz_t, \vz_{t+1}, \dots, \vz_{t+H-1})$ of length $H$ via $M$ iterations of refinement.
At each iteration,
MPPI samples $N$ action sequences from a Gaussian distribution,
$\vz_{t:t+H-1}^{(i)} \sim \gN(\bm{\mu}_{t:t+H-1}, \bm{\Sigma})$ for $i \in [N]$,
where $\bm{\mu}_{t:t+H-1}$ is the current mean parameter and
$\bm{\Sigma}$ is a fixed diagonal covariance matrix.
It then computes the sum of the predicted rewards $\{R^{(i)} := \sum_{j=t}^{t+H-1} \hat{r}_j^{(i)} \}$
using the predicted states from the latent predictive model $\hat{p}_z(\vs'|\vs, \vz)$,
and updates the mean parameter as follows:
\begin{align}
    \bm{\mu}_{t:t+H-1} \gets \frac{\sum_{i \in [N]} e^{\alpha R^{(i)}} \vz^{(i)}_{t:t+H-1}}{\sum_{i \in [N]} e^{\alpha R^{(i)}}}, \label{eq:mppi_update}
\end{align}
where $\alpha$ is a temperature hyperparameter.
After $M$ iterations of refinement, the agent performs only the first latent action
and repeats this process to find the next optimal sequence $(\vz_{t+1}, \vz_{t+2}, \dots, \vz_{t+H})$.
We summarize the full training procedure of MPPI in \Cref{alg:mppi}.

\subsection{MBPO Training Procedure}
\label{sec:appx_training_mbpo}

\begin{algorithm}[t!]
    \caption{MBPO with PMA}
    \begin{algorithmic}[1]
        \STATE Initialize task policy $\pi(\vz|\vs)$, frozen replay buffer $\gD_{\text{frozen}}$, replay buffer $\gD$
        \FOR{$i \gets 1$ to (\# epochs)}
            \STATE $d \gets \texttt{TRUE}$
            \FOR{$j \gets 1$ to (\# steps per epoch)}
                \IF{$d = \texttt{TRUE}$}
                    \STATE Sample $\vs$ from either $\mu(\cdot)$ with a probability of $P$
                    or $\gD_{\text{frozen}}$ with a probability of $(1-P)$
                \ENDIF
                \STATE Sample latent action $\vz \sim \pi(\vz|\vs)$
                \STATE Predict next state $\vs' = \E[\hat{p}_z(\cdot|\vs, \vz)]$
                \STATE Compute predicted reward $r$ and predicted termination $d$ using $\vs$ and $\vs'$
                \IF{(current horizon length) $\geq H$}
                    \STATE $d \gets \texttt{TRUE}$
                \ENDIF
                \STATE Add transition $(\vs, \vz, r, \vs')$ to $\gD$
            \ENDFOR
            \STATE Train task policy with SAC using mini-batches from $\gD$
        \ENDFOR
        \FOR{$i \gets 1$ to (\# evaluation rollouts)}
            \WHILE{not termination}
                \STATE Compute latent action $\vz = \E[\pi(\cdot|\vs)]$
                \STATE Compute action $\va = \E[\pi_z(\cdot|\vs,\vs)]$
                \STATE Get $r$ and $\vs'$ from environment
            \ENDWHILE
        \ENDFOR
    \end{algorithmic}
    \label{alg:mbpo}
\end{algorithm}

MBPO \citep{mbpo_janner2019} is a Dyna-style \citep{dyna_sutton1990} model-based RL algorithm,
which trains a model-free RL method on top of truncated model-based rollouts starting from intermediate environment states.
In our zero-shot setting,
MBPO uses the restored replay buffer from unsupervised training
to sample starting states.
Specifically,
at each epoch,
MBPO generates multiple model-based truncated trajectories
$(\vs_0, \va_0, r_0, \vs_1, \va_1, r_1, \dots, \vs_{H-1}, \va_{H-1}, r_{H-1})$ of length $H$
using the learned predictive model,
where $\vs_0$ is sampled either from the true initial state distribution $\mu(\cdot)$ with a probability of $P$
or from the restored replay buffer $\gD_{\text{frozen}}$ with a probability of $(1-P)$.
It then updates the task policy $\pi(\vz|\vs)$ using the collected trajectories with SAC \citep{sac_haarnoja2018},
and repeats this process.
Note that the restored replay buffer is only used to provide starting states.
Also, if we set $P$ to $1$ and $H$ to the original horizon length $T$,
MBPO corresponds to vanilla SAC trained purely on model-based transitions.
After completing the training of MBPO,
we measure the performance by testing the learned task policy in the true environment.
We summarize the full training procedure of MBPO in \Cref{alg:mbpo}.

\section{Implementation Details}

We implement PMA on top of the publicly released codebase of LiSP \citep{lisp_lu2021}.
We release our implementation at the following repository: {\url{https://github.com/seohongpark/PMA}}.
We run our experiments on an internal cluster consisting of A5000 or similar GPUs.
Each run in our experiments takes no more than two days.

\subsection{Environments}
For our benchmark, we use seven MuJoCo robotics environments from OpenAI Gym \citep{mujoco_todorov2012,openaigym_brockman2016}:
HalfCheetah, Ant, Hopper, Walker2d, InvertedPendulum, InvertedDoublePendulum, and Reacher.
We mostly follow the environment configurations used in \citet{dads_sharma2020}.
We use an episode length of $200$ for all environments.
For HalfCheetah, Ant, Hopper, and Walker2d,
we exclude the global coordinates from the input to the policy.
For Hopper and Walker2d, we use an action repeat of $5$
to have similar discretization time scales to the other environments.
Regarding early termination,
we follow the original environment configurations:
Ant, Hopper, Walker2d, InvertedPendulum, and InvertedDoublePendulum have early termination conditions,
while HalfCheetah and Reacher do not.

\subsection{Tasks}
We describe the reward functions of our $13$ tasks.
We mostly follow the reward scheme of the original task of each environment.
For the environments with early termination, the agent additionally receives a reward of $1$ at every step.

\textbf{HalfCheetah Forward:} The reward is the $x$-velocity of the agent.

\textbf{HalfCheetah Backward:} The reward is the negative of the $x$-velocity of the agent.

\textbf{Ant East:} The reward is the $x$-velocity of the agent.

\textbf{Ant North:} The reward is the $y$-velocity of the agent.

\textbf{Hopper Forward:} The reward is the $x$-velocity of the agent.

\textbf{Hopper Hop:} The reward is the maximum of $0$ and the $z$-velocity of the agent.

\textbf{Walker2d Forward:} The reward is the $x$-velocity of the agent.

\textbf{Walker2d Backward:} The reward is the negative of the $x$-velocity of the agent.

\textbf{InvertedPendulum Stay:} The reward is the negative of the square of the $x$-position of the agent.

\textbf{InvertedPendulum Forward:} The reward is the $x$-velocity of the agent.

\textbf{InvertedDoublePendulum Stay:} The reward is the negative of the square of the $x$-position of the agent.

\textbf{InvertedDoublePendulum Forward:} The reward is the $x$-velocity of the agent.

\textbf{Reacher Reach:} The reward is the negative of the Euclidean distance between the target position and the fingertip.

\subsection{Hyperparameters}

We present the hyperparameters used in our experiments in \Cref{table:hyp,table:hyp_mppi,table:hyp_mbpo}.
For the MPPI results in \Cref{fig:mppi_all},
we individually tune the MOPO penalty $\lambda$ for each method and task (\Cref{table:hyp_mppi}),
where we consider $\lambda \in \{0, 1, 5, 20, 50\}$.
For the MBPO results in \Cref{fig:mbpo_all},
we individually tune the MOPO penalty $\lambda$, the rollout horizon length $H$, and the reset probability $P$
for each method and task (\Cref{table:hyp_mbpo}),
where we consider $\lambda \in \{0, 1, 5, 20, 50\}$ and $(H, P) \in \{(1, 0), (5, 0), (15, 0), (15, 0.5), (50, 0), (50, 0.5), (200, 1)\}$.
We use $(H, P) = (200, 1)$ for the SAC results in \Cref{fig:sac_all}.
When training MBPO in Ant, we additionally apply $\max(0, \cdot)$ to the penalty-augmented reward
to prevent the agent from preferring early termination to avoid negative rewards.

\begin{table}[t]
    \caption{Hyperparameters.}
    \label{table:hyp}
    \vskip 0.15in
    \begin{center}
    \begin{tabular}{lc}
        \toprule
        Hyperparameter & Value \\
        \midrule
        \# epochs & $10000$ \\
        \# environment steps per epoch & $4000$ \\
        \# gradient steps per epoch & $64$ (policy), $32$ (model), $1$ (RND network) \\
        Episode length $T$ & $200$ \\
        Minibatch size & $256$ \\
        Discount factor $\gamma$ & $0.995$ \\
        Replay buffer size & $100000$ \\
        \# hidden layers & $2$ \\
        \# hidden units per layer & $512$ \\
        Nonlinearity & ReLU \\
        Optimizer & Adam \citep{adam_kingma2014} \\
        Learning rate & $3 \times 10^{-4}$ \\
        Target network smoothing coefficient $\tau$ & $0.995$ \\
        Ensemble size & $5$ \\
        Reward scale & $10$ \\
        Latent action dimensionality $|\gZ|$ & $|\gA|$ \\
        PMA, DADS \# samples for entropy approximation $L$ & $100$ \\
        PMA ensemble variance scale $\beta$ & $5$ (Walker2d), $50$ (Hopper), $0.03$ (Otherwise) \\
        CM (Disag.) ensemble variance scale $\beta$ & $0.3$ (Walker2d), $1$ (HalfCheetah, Ant), $3$ (Hopper), $10$ (Otherwise) \\
        MPPI horizon length $H$ & $15$ \\
        MPPI population size $N$ & $256$ \\
        MPPI \# iterations $M$ & $10$ \\
        MPPI temperature $\alpha$ & $1$ \\
        MPPI variance $\Sigma$ & $I$ \\
        \bottomrule
    \end{tabular}
    \end{center}
    \vskip -0.1in
\end{table}

\begin{table}[t]
    \caption{MOPO $\lambda$ for the MPPI results (\Cref{fig:mppi_all}).}
    \label{table:hyp_mppi}
    \vskip 0.15in
    \begin{center}
    \begin{tabular}{lccccc}
        \toprule
        Task & PMA & DADS & CM (Rand.) & CM (Disag.) & CM (RND) \\
        \midrule
        HalfCheetah Forward & $1$ & $1$ & $1$ & $1$ & $1$ \\
        HalfCheetah Backward & $1$ & $1$ & $1$ & $1$ & $1$ \\
        Ant East & $20$ & $20$ & $20$ & $20$ & $20$ \\
        Ant North & $20$ & $20$ & $20$ & $20$ & $20$ \\
        Hopper Forward & $5$ & $5$ & $1$ & $5$ & $5$ \\
        Hopper Hop & $1$ & $1$ & $5$ & $5$ & $5$ \\
        Walker2d Forward & $1$ & $1$ & $1$ & $1$ & $1$ \\
        Walker2d Backward & $1$ & $1$ & $1$ & $1$ & $1$ \\
        InvertedPendulum Stay & $1$ & $1$ & $1$ & $1$ & $1$ \\
        InvertedPendulum Forward & $5$ & $5$ & $5$ & $5$ & $5$ \\
        InvertedDoublePendulum Stay & $0$ & $5$ & $5$ & $5$ & $5$ \\
        InvertedDoublePendulum Forward & $5$ & $5$ & $1$ & $5$ & $1$ \\
        Reacher Reach & $0$ & $0$ & $0$ & $0$ & $0$ \\
        \bottomrule
    \end{tabular}
    \end{center}
    \vskip -0.1in
\end{table}

\begin{table}[t]
    \caption{(MOPO $\lambda$, MBPO $H$, MBPO $P$) for the MBPO results (\Cref{fig:mbpo_all}).}
    \label{table:hyp_mbpo}
    \vskip 0.15in
    \begin{center}
    \begin{tabular}{lccccc}
        \toprule
        Task & PMA & DADS & CM (Rand.) & CM (Disag.) & CM (RND) \\
        \midrule
        HalfCheetah Forward & $(1, 15, 0.5)$ & $(1, 15, 0.5)$ & $(1, 15, 0.5)$ & $(1, 15, 0.5)$ & $(1, 15, 0.5)$ \\
        Ant East & $(50, 15, 0.5)$ & $(10, 15, 0.5)$ & $(50, 15, 0.5)$ & $(10, 15, 0.5)$ & $(50, 15, 0.5)$ \\
        Hopper Forward & $(1, 200, 1)$ & $(1, 200, 1)$ & $(1, 200, 1)$ & $(1, 50, 0.5)$ & $(1, 15, 0.5)$ \\
        Walker2d Forward & $(1, 50, 0)$ & $(1, 50, 0)$ & $(1, 15, 0)$ & $(1, 50, 0)$ & $(1, 15, 0)$ \\
        \bottomrule
    \end{tabular}
    \end{center}
    \vskip -0.1in
\end{table}

\end{document}